\documentclass{article}

\usepackage{arxiv}
\usepackage[utf8]{inputenc}
\usepackage[T1]{fontenc}
\usepackage{amsmath,amssymb,amsthm}
\usepackage{booktabs}
\usepackage{graphicx}
\usepackage{microtype}
\usepackage{xcolor}
\usepackage[colorlinks=true,allcolors=blue!55!black]{hyperref}

\newtheorem{proposition}{Proposition}
\newtheorem{definition}{Definition}

\DeclareMathOperator{\ONL}{ONL}
\newcommand{\PAR}{\mathrm{PAR}}
\newcommand{\MUX}{\mathrm{MUX}}
\newcommand{\bx}{\mathbf{x}}
\newcommand{\R}{\mathbb{R}}

\title{Relocating Nonlinearity: How a Downstream Learner\\
Reshapes What Genetic Programming Must Evolve}

\author{
 Nam H. Le \\
  Vermont Complexity Center\\
  University of Vermont\\
  Burlington, VT 05405 \\
  \texttt{nam.le@uvm.edu}
}

\begin{document}
\maketitle

\begin{abstract}

Genetic programming was conceived as a way of evolving solutions: the program is
the answer, and fitness is the error of its own output. A substantial line of work
instead makes the program an input to a separate learner, so that fitness
measures the learner's output rather than the program's. Which learner to attach
matters and depends on the problem, with no account of what decides it. The
account has to concern the program: what makes a target hard for genetic
programming is how much nonlinearity the program must build by composing
primitives. Any nonlinearity the learner supplies, the program is spared from
building. The quantity to measure is therefore how much of a target's
nonlinearity a learner can take over, and on continuous benchmarks it can only
be estimated.

Here we show that how much the program must still build is decided by which
learner is attached, and is measurable on the programs themselves. Moving to
Boolean domains, where a target's nonlinearity is exactly its Fourier degree, we
tune that degree from one to six with everything else fixed. Conditioned on
success, a linear learner forces the evolved program to the target's degree
exactly, at one through six without a single exception over thirty runs per
setting, while tree ensembles succeed with far simpler programs and four times as
often.

This gives the field two things. Where the target's structure is known or
measurable, the learner can be chosen from that structure rather than from the
target's reputation for difficulty: our control target has \emph{lower} degree
than the parity problems yet gains nothing from a nonlinear learner, while
targets reducible to a simple statistic gain a great deal. Where the structure
is not known, the same quantities serve as diagnostics: the degree of the
evolved program is free to compute during any run and shows whether the
learner is contributing nothing or is carrying the model; on two measured
datasets the resulting procedure picks the right learner before any
evolution is run. The latter comes
with a warning: evolution internalises what the learner supplies in only four of
sixty-three conditions, and grows more dependent on it in forty-four. The more
capable the learner, the less of the model is left legible in the program.

\end{abstract}

\section{Introduction}

Genetic programming (GP) was introduced as a way of evolving computer programs to
solve problems~\cite{koza1992genetic}. The defining arrangement is that the
program is the answer: it is executed on the task, and its fitness is the error
of its own output. Two application areas have carried most of the work. In
symbolic regression the evolved program is the model. The arrangement runs from
early successes in recovering physical laws from data~\cite{schmidt2009distilling},
through a modern line of work combining expression search with neural guidance,
physics-inspired decomposition, or reinforcement
learning~\cite{udrescu2020aifeynman,petersen2021deep,cranmer2020discovering}, to
benchmarking against standard machine learning methods at
scale~\cite{lacava2021srbench}. In classification the evolved program produces
quantities from which a decision is made, in fault
diagnosis~\cite{guo2005feature}, on high-dimensional
data~\cite{tran2019gp}, and in multiclass problems~\cite{munoz2015m3gp}.

A substantial line of work departs from the defining arrangement. The evolved
program stops being the answer. It becomes an input to a separate, deterministic
learner that is refitted for every candidate, and fitness measures the learner's
output instead of the program's. In classification this is usually called feature
construction~\cite{krawiec2002gp,guo2005feature,tran2019gp,munoz2015m3gp,lacava2019m4gp};
in regression it appears as multi-gene GP, which fits a linear model over several
evolved trees~\cite{searson2010gptips}, multiple regression GP, which recombines a
program's subexpressions by least squares before
selection~\cite{arnaldo2014mrgp}, evolutionary feature
synthesis~\cite{arnaldo2015efs}, and systems that learn feature spaces or networks
of trees with the downstream model's weights as
feedback~\cite{lacava2019feat,lacava2020feature}. The difference from standard
genetic programming is not incidental. Once fitness stops measuring the program's
own output, the program stops being a solution and becomes one component of an
answer completed elsewhere, a \emph{representation}. The literature does not put
it this way, and so never asks the question that follows: what is the program now
responsible for?

The most recent instance of the arrangement is also its simplest, and makes the
question unavoidable. One evolved program supplies a single feature; its values on the data form one
semantic vector, and a univariate learner is fitted from that feature to the
target. Fitness is the error of the composition, and the evolutionary loop is
left entirely unchanged~\cite{le2026laser}, an arrangement we refer to below as
LaSER. LaSER is the author's own earlier system; the present paper is an
analysis of why and when it works, and duplicates none of its content. Because only the learner varies, the effect of the
learner can be read directly. That work reports a linear learner dominating one
benchmark suite while nonlinear learners take over on a harder one, and reports
that the evolved program sometimes attains the target unaided, which it reads as a
Baldwin effect. Neither result is explained. The first says which learner won
where without naming any property of the targets that would transfer to a new
problem, and the second is measured only at the end of a run on the easiest
problems.

What the learner contributes is therefore still unknown. Benchmark comparison
cannot settle it, because such a comparison characterises a system, not a
problem. The account has to concern the program. What makes a target hard for
genetic programming is, to a large extent, how much nonlinearity the program must
build by composing primitives: a deeply nested product of the right variables is
hard to assemble in a way that a rescaling of it is not. Whatever nonlinearity the
attached learner can supply, the program does not have to build. The question of
which learner to attach is therefore the question of how much of a target's
nonlinearity a learner can take over, and the quantity that answers it is the
target's \emph{order of nonlinearity} in the sense of Vladislavleva et
al.~\cite{vladislavleva2009order}. The difficulty is that on the continuous
benchmarks where this work is done that quantity must be estimated, and the
estimate moves with the sampling range, the error tolerance and the polynomial
basis. A claim of the form ``a nonlinear learner is needed when the order of
nonlinearity is high'' cannot be tested cleanly while the quantity inside it is
partly an artefact of the experimenter's choices. The question has stayed open
because it is a measurement problem before it is a scientific one.

We treat it as such. On $\{-1,+1\}^n$ every function has a unique multilinear
expansion, and because $x_i^2=1$ the monomials of degree at most $d$ span
precisely the multilinear terms of degree at most $d$; order of nonlinearity
therefore coincides with Fourier degree and is computable in closed form by a
Walsh-Hadamard transform~\cite{odonnell2014analysis}. The parity family supplies a
dial that moves that degree from one to six while the input space, the data, the
primitive set and every search parameter stay fixed. The domain is an instrument
chosen to make a measurement possible, not the subject of the
claim. The measurement itself is the departure: instead of inferring the division
of labour from performance, we compute the order of nonlinearity of the program
each run actually produced, a quantity that to our knowledge has not been
reported for any system of this kind.

Here we show two things: which learner is attached decides how much the program
must still build, and that amount is directly visible in the programs evolution
produces. Conditioned on success, a linear learner pins the order of the
evolved program to the target's exactly, at one through six with no exception in
any of the thirty runs per condition, while tree ensembles succeed with programs
of order two or three; success rates separate accordingly, at six and seven runs
of thirty against twenty-two, nineteen and twenty-six at the highest order. The
advantage is real but bounded, and the bound is informative: a $6$-multiplexer
control, of \emph{lower} order than the parity problems where the separation is
largest, gains nothing at all from a nonlinear learner, so difficulty and the
appropriate learner are orthogonal. Neither component approaches the result alone: the same learners reach $0.112$
held-out accuracy on the raw inputs, and evolution without a learner collapses to
a constant program.

A second finding follows from reading the arrangement literally. The learner is
refitted every generation and its fitted parameters are never inherited, which is
the structure of lifetime learning in the Baldwinian
sense~\cite{Baldwin:1896,Hinton:1987}. Following the evolved program across
generations shows that evolution does not internalise what the learner supplies.
Across sixty-three problem and learner combinations, genetic assimilation (the
genotype coming to replicate, on its own, what learning used to supply) occurs
in four, three of which are canonical genetic programming, while in forty-four the
contribution of learning widens over evolutionary time; on parity of order four
the innate accuracy of the best individual falls from $0.713$ to $0.394$ while its
learned accuracy rises to $1.000$. The more capable the learner, the less of the
model lies in the program one can read, which places predictive accuracy in
conflict with the self-sufficiency of symbolic models that motivates the use of
genetic programming in the first place.

The contributions are as follows. (1) We make the division of labour between
the evolved program and the learner measurable. The measure is the order of
nonlinearity: the least polynomial degree that reproduces a function to a
stated tolerance. On Boolean inputs this equals the Fourier degree, which is
exact, with nothing left to estimate. (2) We prove four simple statements
about it. An affine learner adds no order, so with linear scaling or a least
squares output stage the program must build all of the target's nonlinearity.
Parity of order $k$, and in fact every symmetric Boolean function, can be
written as an order-$1$ program followed by one fixed nonlinear map, so the
targets that reward an expressive learner form a precise class, not a
collection of individual cases. (3) We measure the order of the programs that
evolution actually produces, and find that it follows what the attached
learner can absorb, not what the target contains; a control target of lower
order gains nothing from an expressive learner, so being hard and being
decomposable are different properties; matched random juntas of the same order
as $\PAR_4$, and a multi-feature variant of the arrangement, test the same
reading from two further directions. (4) We show that the more the learner
absorbs, the less evolution writes into the program itself, which reverses
the usual reading of the Baldwin effect in this setting.

\section{Related Work}

Machine learning has entered genetic programming in two quite different places,
and the distinction matters for what follows. In the first, learning is put into
the \emph{search}: the evolved expression remains the answer, and a learned
component decides where to look for it. Deep symbolic regression trains a
generator of expressions by risk-seeking policy
gradients~\cite{petersen2021deep}; neural-guided seeding initialises a GP
population from a trained model~\cite{mundhenk2021symbolic}; AI Feynman uses a neural
network to detect symmetry and separability and to recurse on
subproblems~\cite{udrescu2020aifeynman}; and graph networks trained on a system
are distilled into closed-form expressions after the
fact~\cite{cranmer2020discovering}. A large benchmarking effort now places these
alongside standard machine learning
methods~\cite{lacava2021srbench}. In all of them the arrangement Koza defined is
preserved, because fitness still measures what the expression itself computes. The
present paper is not about this family. The distinction matters because the phrase
``genetic programming combined with machine learning'' now most often refers to
it.

In the second place, learning is put into the \emph{model}: as in Section~1, the
evolved program gives up being the answer, becoming instead an input to a learner
that is refitted for every candidate, with fitness measuring the learner's
output. In classification this is the standard framing of feature construction,
from Krawiec's formulation of it as changing the representation of the input data
for a machine learner~\cite{krawiec2002gp}, through feature generation for fault
diagnosis~\cite{guo2005feature} and multiple-feature construction on
high-dimensional data~\cite{tran2019gp}, to multidimensional transformations
consumed by a distance based classifier~\cite{munoz2015m3gp,lacava2019m4gp}. In
regression the same shift appears as multi-gene
GP~\cite{searson2010gptips}, multiple regression
GP~\cite{arnaldo2014mrgp}, evolutionary feature
synthesis~\cite{arnaldo2015efs}, and systems that learn feature spaces or networks
of trees using the downstream model's weights as
feedback~\cite{lacava2019feat,lacava2020feature,lacava2019semantic}, the same
systems already sketched in Section~1. Icke and
Bongard combined GP based symbolic regression with deterministic learners with the
same motivation~\cite{icke2013improving}.

Two features of that second family are relevant here. The first is what it holds
fixed. Whatever the domain, the learner is refitted per candidate and never
inherited, and the evolved program is scored by proxy, so in every one of these
systems the program is a representation, not a solution, though none of them
says so. The second is what it varies, and here the variation is
narrower than it first appears: the overwhelming majority attach a \emph{linear}
downstream model, whether a least squares combination of
genes~\cite{searson2010gptips,arnaldo2014mrgp}, a linear map over evolved
features~\cite{arnaldo2015efs}, or a linear or distance based decision
rule~\cite{munoz2015m3gp,lacava2019m4gp}. By Proposition~\ref{prop:affine} below,
such a stage supplies no order of nonlinearity at all, so in all of this work the
evolved features must between them carry the target's full nonlinear structure.
This literature has not overlooked the effect we report; it has been
structurally unable to observe it, because it has almost never varied the one
component that would reveal it.

A third strand attacks the same difficulty from inside the evolutionary loop, not
after it. Linear scaling fits a least squares affine map to each
individual's output so that evolution is scored on shape rather than
scale~\cite{keijzer2003improving}, and is best read as the degenerate member of
the second family above: one evolved feature, one affine learner. Semantic methods
instead change the variation operators so that behaviour moves geometrically in
semantic space~\cite{moraglio2012geometric,vanneschi2014geometric}, at the cost of
rapid code growth~\cite{vanneschi2014survey}. These are not merely empirical
proposals; geometric semantic operators come with a conical fitness landscape
argument, and semantic backpropagation and library based methods have a developed
theory. Our results touch this strand at one point only: stating what an affine
post-hoc map cannot contribute locates linear scaling exactly, instead of
treating it as merely a weak learner.

A recent line of work varies the \emph{placement} of the learned component
rather than its capacity. Huang et al.\ equip linear genetic programming
with tunable primitives, fitted by least squares or gradient steps during
evolution, at three positions along the program: tunable terminals at the
input, which regress over ranges of spectral features; tunable functions in
the interior, which reshape one intermediate value each; and a multivariate
linear output stage over several registers at the
end~\cite{huang2026tunable}. On noisy, data-poor fish biomass spectra,
adding the three placements cumulatively improves prediction while the
programs stay compact. In the vocabulary of this paper the three placements
are not equivalent. The output stage is exactly the affine downstream
learner of the second family, and Proposition~\ref{prop:affine} states its
budget: combination, scale and offset, but no order of nonlinearity. The
interior and input primitives relocate small parametric learners into the
program itself, each contributing a fixed amount of tunable response around
the structure evolution supplies. That reading is consistent with where
tunable primitives help most, targets whose difficulty is numerical (noise,
heterogeneous scales) rather than structural, and it makes placement and
capacity complementary axes: their study varies where a bounded learner
sits, ours varies how much order a single end-of-pipeline learner can
absorb.

The measure this paper builds on also comes from within genetic programming.
Vladislavleva et al.\ proposed the minimal degree of a Chebyshev approximation as
a complexity measure for evolved models, and used it as a selective pressure
inside a Pareto GP to keep solutions simple~\cite{vladislavleva2009order}. Their
measure applies to the model and acts as a constraint on it. We use the same
notion twice over in the opposite direction: as a property of the \emph{target},
predicting which downstream learner is required, and as a measurement made
\emph{on} the program a run produced, not a pressure imposed on it. Neither
use appears in that literature, and the second is what supplies the evidence in
Section~\ref{sec:reloc}.

The experimental instrument likewise has a long history in a different role.
Boolean function synthesis is among the oldest GP benchmarks and the even parity
family is its canonical hard case~\cite{koza1992genetic,poli2008field}. Its role
has been diagnostic: because parity admits no useful partial solutions, it reveals
whether a proposed mechanism supplies structural search or merely smooths a
numerical landscape, which is why it appears as the first hard case in the work
introducing geometric semantic GP~\cite{moraglio2012geometric}. In every such use
it is a single hard instance, chosen for being hard. The property exploited here
is different: parity is an indexed \emph{family} whose index is exactly the
quantity under study, which turns order of nonlinearity from a property one
estimates into an experimental variable one sets. That a Boolean function has a
well defined Fourier degree is standard~\cite{odonnell2014analysis}; what is new
is using it to measure evolved programs.

The second half of the paper connects to a separate tradition. That lifetime
learning can redirect evolution without Lamarckian inheritance is the Baldwin
effect~\cite{Baldwin:1896}, demonstrated computationally by Hinton and
Nowlan~\cite{Hinton:1987} and studied since in evolutionary
computation~\cite{whitley1994lamarckian,Le2018}. Mayley identified the two
opposing forces that decide its outcome: learning smooths the fitness landscape,
which helps evolution, but it also shields genotypes from selection, removing the
pressure to internalise what learning already supplies, so assimilation occurs
only when the cost of learning outweighs the
shielding~\cite{mayley1996landscapes}. Every system in the second family above is
an instance of this structure, with a learning episode in each individual's
lifetime whose results are discarded before reproduction, yet none is analysed in
these terms. The consequence is that a question with a long history in
evolutionary computation has never been asked of the systems that most cleanly
instantiate it.

Closest to the present work is a recent proposal of the single-feature form of the
arrangement, evaluated on continuous symbolic regression
benchmarks~\cite{le2026laser}. It is the only member of the second family we
know of in which the downstream learner is varied systematically while everything
else is held fixed, which is what makes its two reported results informative and
the absence of an account of them worth addressing. The first, that a linear learner dominates
one benchmark suite while nonlinear learners take over on a harder one, names no
property of the targets and therefore transfers to no new problem; we supply that
property and show that the two suites differ in it. The second, that the evolved
program sometimes attains the target unaided, is read there as evidence of a
Baldwin effect, but it is measured only at the end of a run and only on the
easiest problems; we measure the same phenomenon at every generation across the
full range of target order and find that it does not generalise, the prevailing
direction being the opposite one.

\section{Method}

\subsection{Problem formulation}

Throughout, the data are $m$ pairs of an input $\bx$ and a target $y$, and the
task is to predict $y$ from $\bx$. Figure~\ref{fig:pipeline} puts the two
arrangements side by side. In canonical genetic programming the evolved program is
executed on the inputs and its own output is compared with the target, so the program has to build the whole of the
target's structure by composing primitives. In the arrangement studied here the
program's output is instead handed to a learner $f$, which is fitted afresh for
every candidate, and it is the learner's output that is compared with the target.
As in Section~1, the program is now one piece of an answer finished elsewhere,
and how much of the structure it still has to build depends on what $f$ can
supply.

\begin{figure}[ht]\centering
\includegraphics[width=\linewidth]{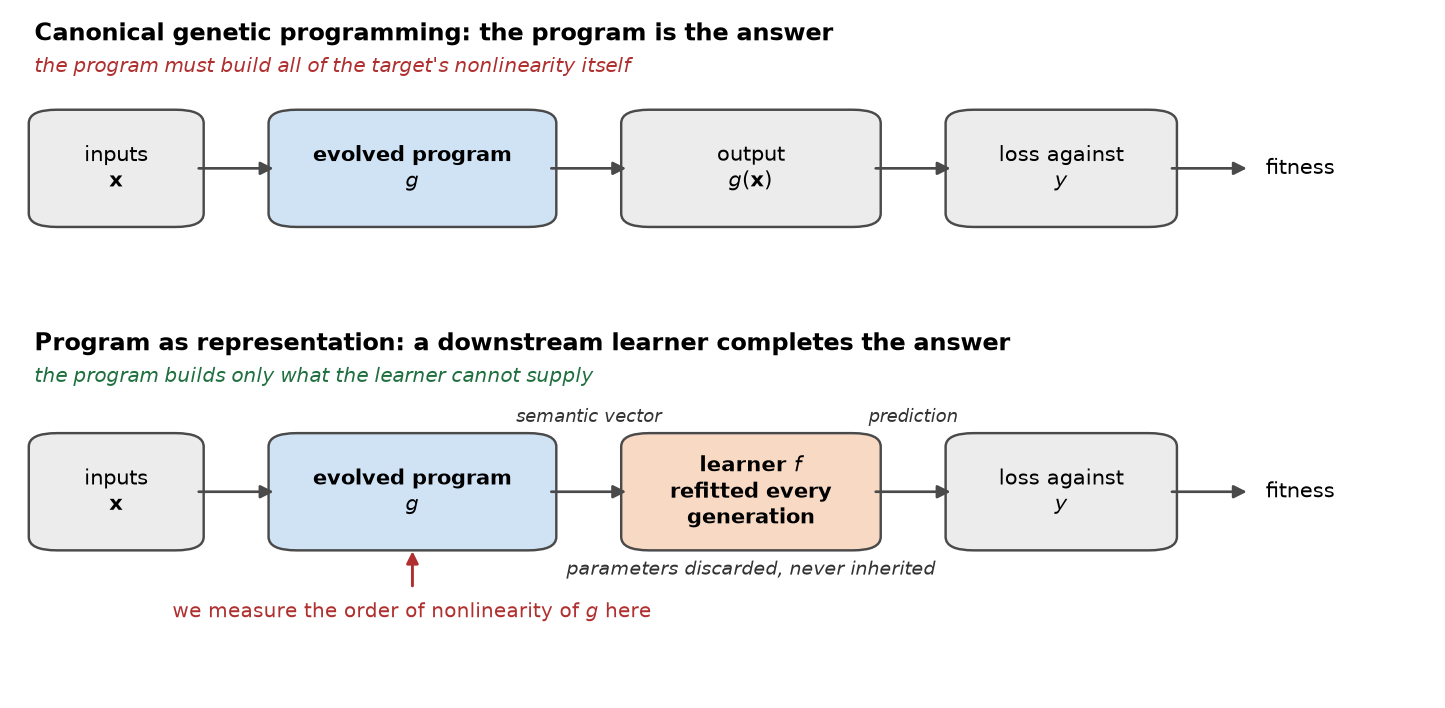}
\caption{The two arrangements, with $\bx$ an input, $y$ its target, $g$ the evolved
program and $f$ the downstream learner. Only the second row is studied here. The learner is
refitted for every candidate in every generation and its fitted parameters are
discarded, so nothing it learns is inherited. Everything in the evolutionary loop
is held fixed across conditions; only $f$ changes.}
\label{fig:pipeline}
\end{figure}

Formally, the data are the pairs $(\bx_1,y_1),\dots,(\bx_m,y_m)$, with each input
$\bx_i$ in an input space $\mathcal{X}$ and each target $y_i$ a real number. An
individual encodes a function $g:\mathcal{X}\to\R$, and its values on the inputs,
$(g(\bx_1),\dots,g(\bx_m))$, are its \emph{semantic vector}. Fitness is
\begin{equation}
\mathcal{L}(f\circ g)=\frac1m\sum_{i=1}^m\bigl(f(g(\bx_i))-y_i\bigr)^2,
\qquad f=\arg\min_{f'\in\mathcal{F}}\mathcal{L}(f'\circ g),
\label{eq:fit}
\end{equation}
where $\mathcal{F}$ is the class the learner is drawn from. Taking $\mathcal{F}$
to be the identity map recovers the first row of Figure~\ref{fig:pipeline};
restricting it to affine maps recovers linear
scaling~\cite{keijzer2003improving} and, in the single feature case, the least
squares combination used by multi-gene and multiple regression
GP~\cite{searson2010gptips,arnaldo2014mrgp}. This is the whole experimental
design: $\mathcal{F}$ is the only thing that differs between conditions, while
selection, variation, the primitive set, the size limit and the seed sequence stay
identical, so any difference we report is attributable to the learner and to
nothing else.

\subsection{Benchmark targets and order of nonlinearity}
\label{sec:domain}

The experiments need a domain in which the order of nonlinearity of a target is a
known quantity rather than an estimate, and in which it can be varied on its own
while everything else is held still. Boolean functions give both. We work on
$\{-1,+1\}^n$ with $n=8$, using the full truth table of $256$ rows split $70/30$
into training and held out data, so that no sampling decision enters anywhere.

Order of nonlinearity is the degree of the simplest polynomial that reproduces a
function to within a tolerance.

\begin{definition}[Order of nonlinearity]
\label{def:onl}
For $h:\mathcal{X}\to\R$ and tolerance $\varepsilon\ge0$, $\ONL_\varepsilon(h)$ is
the least total polynomial degree $d$ for which some polynomial of degree $d$
approximates $h$ with squared error at most $\varepsilon\cdot\mathrm{Var}(h)$.
\end{definition}

This single definition is used throughout, with one convention. On Boolean
domains we take $\varepsilon=0$ and the quantity is exact, equal to the
Fourier degree by Proposition~\ref{prop:fourier} below. On continuous domains
exactness is unavailable and we take the relative tolerance
$\varepsilon=10^{-4}$; the threshold is relative to $\mathrm{Var}(h)$, so the
measure is scale free. Whenever we write $\ONL$ without a subscript, the
domain determines which of the two conventions applies.

\begin{proposition}[Exactness on the hypercube]
\label{prop:fourier}
Let $h:\{-1,+1\}^n\to\R$ have multilinear expansion
$h=\sum_{S\subseteq[n]}\hat h(S)\chi_S$ with $\chi_S(\bx)=\prod_{i\in S}x_i$. Then
$\ONL_0(h)=\max\{|S|:\hat h(S)\neq0\}$, the Fourier degree of $h$.
\end{proposition}

\begin{proof}
Since $x_i^2=1$ on the cube, every monomial of degree at most $d$ reduces to a
multilinear monomial of degree at most $d$, so the span of the degree-$d$
polynomials equals the span of $\{\chi_S:|S|\le d\}$. The multilinear expansion is
unique, so $h$ lies in that span if and only if $\hat h(S)=0$ for all $|S|>d$.
\end{proof}

\noindent\textbf{Remark.} In words: on inputs that are $\pm1$, squares
disappear ($x_i^2=1$), so polynomial degree and Fourier degree are the same
thing, and the order of nonlinearity can be read off a Walsh--Hadamard
transform instead of being estimated. \textbf{Example.} On two bits,
$h=x_1x_2$ has $\ONL_0=2$ (its spectrum is the single set $\{1,2\}$), while
$h=x_1+x_2$ has $\ONL_0=1$; no choice of basis or tolerance is involved in
either statement.

Figure~\ref{fig:domain}(a) is Proposition~\ref{prop:fourier} applied to the targets
we use. Each row is a target and each column a degree, and the cell holds the share
of that target's spectrum sitting at that degree. The numbers come from a
Walsh-Hadamard transform with $\varepsilon=0$ and no basis to choose, so the order
printed beside each row is a fact about the function rather than a reading taken
at some tolerance.

\begin{figure}[ht]\centering
\includegraphics[width=\linewidth]{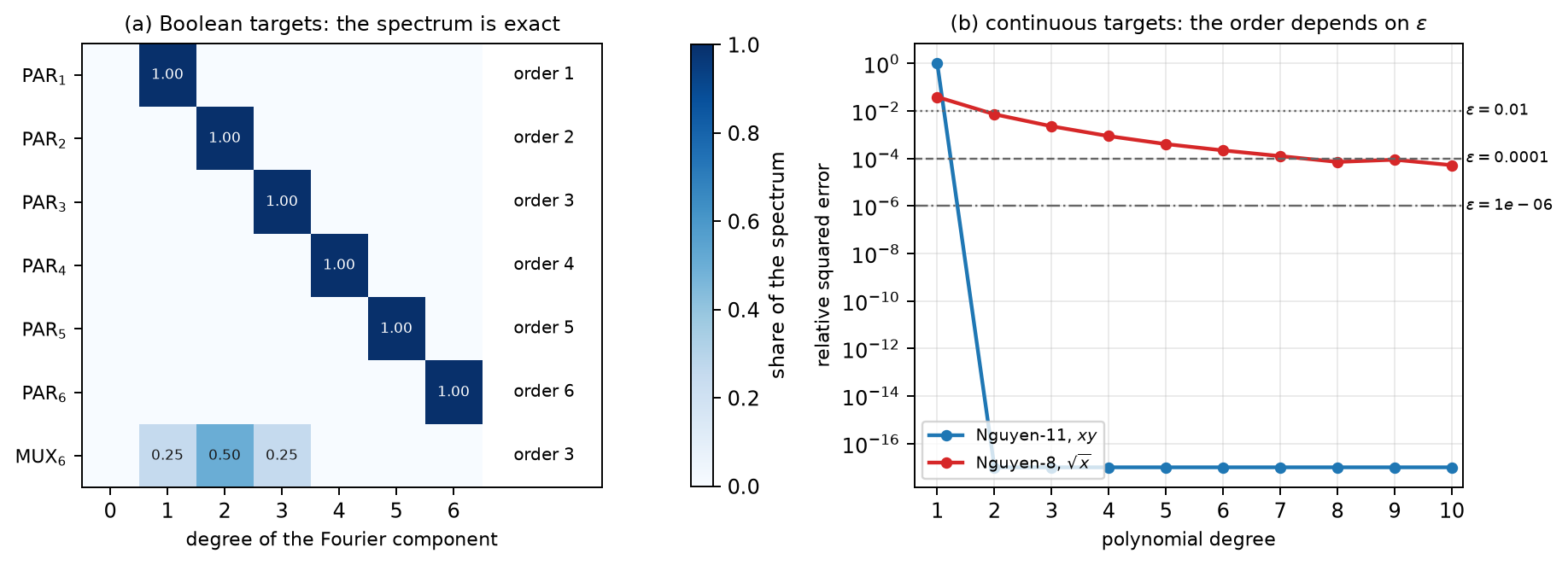}
\caption{Order of nonlinearity of the benchmark targets. (a) The Fourier spectrum of
each Boolean target. Parity puts its entire spectrum at one degree, which is what
makes $k$ a dial; the multiplexer spreads its spectrum over degrees one to three.
(b) The same quantity for the two continuous targets, read off the error curve of
the best polynomial fit. Nguyen-11 is exact at degree two, but for Nguyen-8 the
order is two at $\varepsilon=10^{-2}$, eight at $\varepsilon=10^{-4}$ and
undefined at $10^{-6}$, so on a continuous domain the answer is partly a choice.}
\label{fig:domain}
\end{figure}

Two features of panel (a) set up the experiments. Parity,
$\PAR_k(\bx)=\prod_{i=1}^k x_i$, puts its whole spectrum at degree $k$ and nowhere
else, so sweeping $k$ from one to six moves the target's order through six values
while the input space, the data, the primitive set and every search parameter stay
fixed. We take parity over the first $k$ of the eight inputs, so a run must also work out which inputs matter. The
$6$-multiplexer, in which two address inputs select one of four data inputs, is
the control. Its order is $3$, \emph{lower} than that of $\PAR_4$ through
$\PAR_6$, and its spectrum is spread rather than concentrated; as
Section~\ref{sec:relocate} explains, it is also not symmetric, which is what makes
it the case where a capable learner should buy nothing.

Panel (b) shows what is lost outside this domain, and why the standard benchmarks
cannot serve as the instrument. There the order has to be read off an error curve,
and for a target such as $\sqrt{x}$ the curve descends slowly, so the answer is
two, eight or undefined depending on where the tolerance is drawn. We still use
continuous targets as well, sampled at $1000$ points and split in the same
proportion, but only to check in Section~\ref{sec:bounds} that the ordering
found on the cube survives outside it. Two of them, Nguyen-11 ($xy$) and
Nguyen-8 ($\sqrt{x}$), sit at the extremes of that reading with
$\varepsilon=10^{-4}$ and are examined in detail; the rest of the Nguyen
suite, with the order of each member estimated the same way, and an
eleven-input multiplexer as a second asymmetric control, complete the
comparison there.

\subsection{Order preservation under an affine learner}

A linear learner is, in a specific and useful sense, no help at all; the
proposition below makes that precise.

\begin{proposition}[Affine learners are order preserving]
\label{prop:affine}
Let $f(z)=\alpha z+\beta$ with $\alpha\neq0$. Then
$\ONL_\varepsilon(f\circ g)=\ONL_\varepsilon(g)$ for every $g$. Consequently a
pipeline with an affine learner attains relative error $\varepsilon$ on a target
$y$ only if the evolved program satisfies
$\ONL_\varepsilon(g)\ge\ONL_\varepsilon(y)$.
\end{proposition}

\begin{proof}
If a polynomial $p$ of degree $d$ approximates $g$ to relative error $\varepsilon$
then $\alpha p+\beta$ has the same degree and approximates $\alpha g+\beta$ to the
same relative error; the converse follows with $\alpha^{-1}$. Hence the two
minimal degrees coincide. If in addition $f\circ g$ approximates $y$ to relative
error $\varepsilon$, then
$\ONL_\varepsilon(y)\le\ONL_\varepsilon(f\circ g)=\ONL_\varepsilon(g)$.
\end{proof}

\noindent\textbf{Remark.} In words: scaling and shifting a function does not
change its degree, so an affine learner can fix the scale and the offset of
what the program produces but cannot add any order of nonlinearity to it.
\textbf{Example.} Let the target be $y=x_1x_2$ (order $2$). If the evolved
program is $g=x_1+x_2$ (order $1$), no $\alpha,\beta$ make $\alpha g+\beta$
equal $y$: the pipeline fails however the learner is fitted. If instead
$g=x_1x_2$, then $f(z)=z$ already succeeds. Under an affine learner the
program must reach the target's order on its own.

Figure~\ref{fig:affine} illustrates the proposition, in the same way that a plot of
a non-separable problem illustrates the limits of a perceptron. Plot the
target against the value of the evolved program: an affine learner is a straight
line in that plane and nothing else, while a learner of capacity $V$ is any
function of the program with at most $V$ monotone pieces. Take $\PAR_4$ and two
programs a run might produce. On the left the program is the degree-$1$ sum of the
four relevant inputs, which sorts every input into one of five groups with the
target alternating between $+1$ and $-1$ across them. The best affine learner is
the flat line and explains none of the variance; a learner with four monotone
pieces recovers the target exactly. On the right the program is the degree-$4$
product. Now the affine learner is exact, but only because the program has already
built the whole of the target by itself. The target and the two learner classes are the
same in both panels, so the difference lies entirely in what the program supplies.

\begin{figure}[ht]\centering
\includegraphics[width=\linewidth]{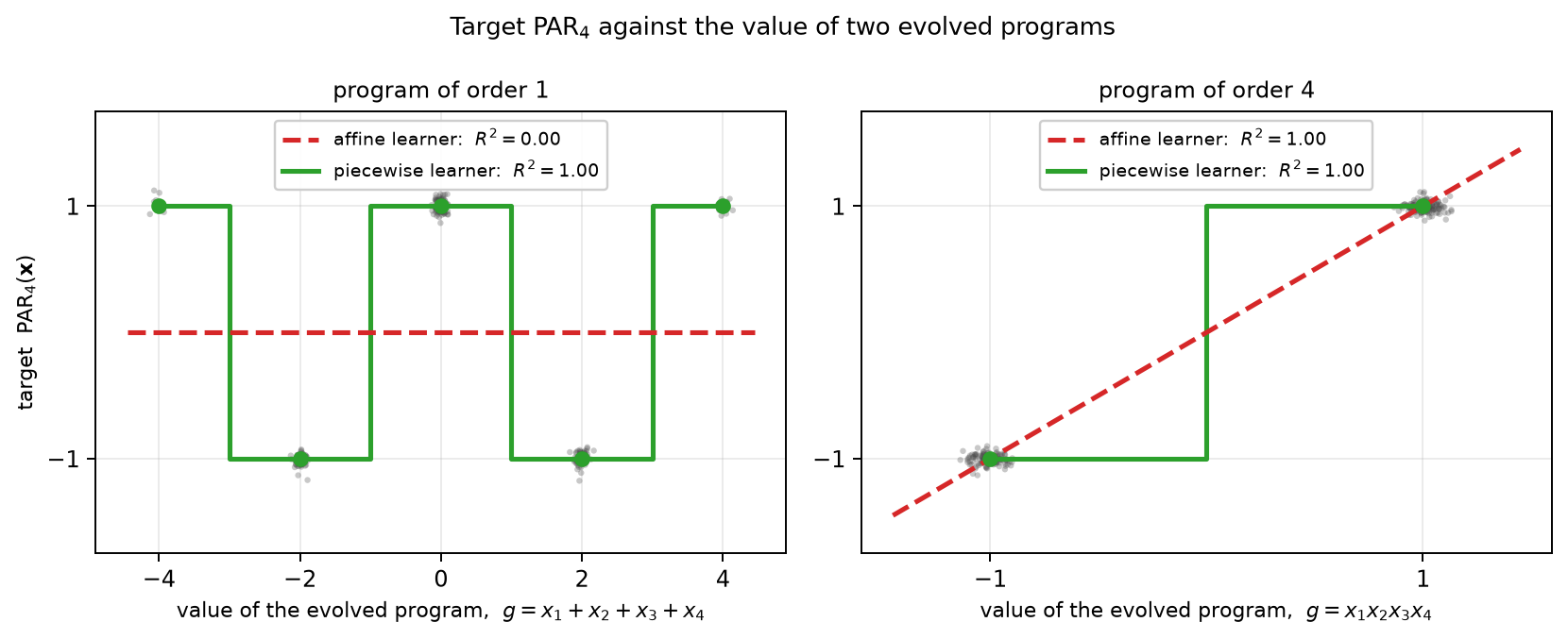}
\caption{Illustration of Proposition~\ref{prop:affine}. The horizontal axis
is the value of the evolved program $g(\bx)$, a single number per input: all
$2^8$ inputs are projected onto this one evolved feature, which is exactly
what the downstream learner sees (left: $g$ is the degree-$1$ sum
$x_1{+}x_2{+}x_3{+}x_4$, taking the five values $-4,\dots,4$; right: $g$ is
the degree-$4$ product). The vertical axis is the target $\PAR_4$; points are
jittered so that coincident inputs are visible. An affine learner can only draw a straight line
here. Left: a program of order $1$, on which the affine learner is useless and a
four-piece learner is exact. Right: a program of order $4$, on which the affine
learner is exact because the program has already done the work. The affine learner
never changes the order of what reaches it; it only decides whether the order the
program supplies is enough.}
\label{fig:affine}
\end{figure}

Rescaling a function does not make it more nonlinear, so a linear learner leaves
the program exactly where it started: whatever nonlinearity the target has, the
program must build all of it. This is why linear scaling and every least squares
combination of evolved features leave the full order on the program's side,
and why, as Section~2 notes, most of the feature
construction literature could not have observed the effect reported here.

\subsection{Relocatable order of nonlinearity}
\label{sec:relocate}

For a nonlinear learner the corresponding question is how far below the target's
order the program is allowed to sit. The answer cannot be stated without bounding
the learner, because an unconstrained $f$ makes the question trivial: a generic
real valued $g$ is injective on a finite sample, so an unconstrained $f$ could
memorise the target from any $g$ whatsoever. We therefore measure the capacity of
a univariate learner by its number of monotone pieces, and write $\mathcal{F}_V$
for the maps that are piecewise monotone with at most $V$ pieces. Every learner we
use has a finite such $V$ on a finite sample; memorisation would need $V$ of order
$2^n$.

\begin{definition}[Relocatable order]
\label{def:delta}
\[
\ONL^{\mathrm{dec}}_{\varepsilon,V}(y)=\min\bigl\{\ONL_\varepsilon(g):
\exists f\in\mathcal{F}_V,\ \|f\circ g-y\|^2\le\varepsilon\,\mathrm{Var}(y)\bigr\},
\qquad
\Delta_{\varepsilon,V}(y)=\ONL_\varepsilon(y)-\ONL^{\mathrm{dec}}_{\varepsilon,V}(y).
\]
\end{definition}

$\ONL^{\mathrm{dec}}$ is the order the program has to reach, given a learner of
capacity $V$. We call $\Delta$ the target's \emph{relocatable order}: how much of
the target's order the learner can take over.
Proposition~\ref{prop:affine} says $\mathcal{F}_1$ gives $\Delta=0$ by
construction. The claim tested in Section~4 is that $\Delta$, and not
$\ONL(y)$, decides whether an expressive learner is worth attaching.

Parity is the case in which $\Delta$ is as large as it can be.

\begin{proposition}[Parity decomposes at order one]
\label{prop:parity}
Let $\PAR_k(\bx)=\prod_{i=1}^k x_i$ on $\{-1,+1\}^n$. Then $\ONL_0(\PAR_k)=k$, yet
there exist $g$ with $\ONL_0(g)=1$ and $f\in\mathcal{F}_k$ with
$f(g(\bx))=\PAR_k(\bx)$ exactly, so $\Delta_{0,k}(\PAR_k)=k-1$.
\end{proposition}

\begin{proof}
The first claim follows from Proposition~\ref{prop:fourier}, the Fourier expansion
of $\PAR_k$ being supported on the single set $[k]$. For the second, take
$g(\bx)=s(\bx)=\sum_{i=1}^k x_i$, of degree $1$. With
$m=|\{i\le k:x_i=-1\}|$ we have $s=k-2m$, hence $m=(k-s)/2$ and
$\PAR_k(\bx)=(-1)^m=\cos\bigl(\pi(k-s)/2\bigr)$. As $s$ takes only the $k+1$
values $\{-k,-k+2,\dots,k\}$, the map $f(s)=\cos(\pi(k-s)/2)$ is realisable as a
piecewise monotone function with at most $k$ pieces.
\end{proof}

Here and throughout, $n$ is the number of input bits available to the program
($n=8$ in all Boolean experiments) and $k\le n$ is the order of the target;
the remaining $n-k$ bits are irrelevant to the target but visible to
evolution.

\noindent\textbf{Remark.} In words: a program working alone must reach
degree $k$, since parity's whole spectrum sits there; but the degree-$1$
sum of the $k$ relevant bits already determines the parity, and one fixed
alternating map of that sum, with at most $k$ monotone pieces, finishes
the job. \textbf{Example.} For $\PAR_4$ the program need only add up the four
relevant inputs, $g(\bx)=x_1+x_2+x_3+x_4$, a degree-$1$ expression. The sum
determines the number of $-1$s among the four bits, and the learner reads
parity off the five possible values: $s=-4$ (four $-1$s, parity $+1$),
$s=-2$ (three, parity $-1$), $s=0$ (two, parity $+1$), $s=+2$ (one, parity
$-1$), $s=+4$ (none, parity $+1$), alternating between $+1$ and $-1$. The same target read through
an affine learner requires the program to build the full four-way product. The two
routes to the identical target differ in the order the program must reach by three,
and the primitive set we use, $\{+,-,\times\}$, expresses both, so the experiment
does not push evolution towards either.

Parity might look like a special case constructed for the purpose. It is not.

\begin{proposition}[All symmetric functions decompose at order one]
\label{prop:symmetric}
Let $h:\{-1,+1\}^n\to\R$ be invariant under permutation of its arguments. Then
there are $g$ with $\ONL_0(g)=1$ and $f\in\mathcal{F}_n$ with $f(g(\bx))=h(\bx)$
exactly, whatever the Fourier degree of $h$.
\end{proposition}

\begin{proof}
A permutation invariant function depends on $\bx$ only through the multiset of its
coordinates, hence on the cube only through $s(\bx)=\sum_i x_i$, of degree $1$.
Define $f$ on the $n+1$ attainable values of $s$ by $f(s)=h(\bx)$ for any $\bx$
with that sum; this is well defined by symmetry and piecewise monotone with at
most $n$ pieces.
\end{proof}

\noindent\textbf{Remark.} In words: a symmetric function cannot tell its
inputs apart, so it only sees how many of them are $+1$, and the degree-$1$
sum carries exactly that information; one fixed map of the sum then finishes
the job, whatever the function's degree. \textbf{Example.} Majority of three
bits has full Fourier degree $3$ (its expansion is
$\tfrac12(x_1+x_2+x_3-x_1x_2x_3)$), yet $g=x_1+x_2+x_3$ with
$f(s)=\mathrm{sign}(s)$ computes it exactly: an order-$1$ program and a
one-split learner replace a degree-$3$ construction.

Majority, threshold and counting functions are all covered, many of them of full
Fourier degree $n$, so the class of targets on which a capable learner should pay
off is a describable class rather than a list of examples.

Nor is the decomposition confined to symmetric targets, provided the learner
has enough pieces.

\begin{proposition}[Juntas decompose at order one, at exponential capacity]
\label{prop:junta}
Let $h:\{-1,+1\}^n\to\R$ depend only on the coordinates $i_1<\dots<i_k$. Then
there is a $g$ with $\ONL_0(g)=1$ and an $f\in\mathcal{F}_{2^k}$ with
$f(g(\bx))=h(\bx)$ exactly, so $\Delta_{0,2^k}(h)=\ONL_0(h)-1$.
\end{proposition}

\begin{proof}
Take $g(\bx)=\sum_{j=1}^{k}3^{\,j-1}x_{i_j}$, of degree $1$. Two assignments of
the relevant bits that differ at highest index $j$ give sums that differ by
$2\cdot3^{\,j-1}$ at that position, while the lower positions can move either
sum by at most $\sum_{i<j}2\cdot3^{\,i-1}=3^{\,j-1}-1$, so $g$ is injective on
the $2^k$ relevant assignments. Define $f$ on the $2^k$ attained values by
$f(g(\bx))=h(\bx)$; a map on $2^k$ points is piecewise monotone with at most
$2^k$ pieces.
\end{proof}

\noindent\textbf{Remark.} The contrast with
Propositions~\ref{prop:parity} and~\ref{prop:symmetric} lies in the capacity,
and it is exponential against linear. Parity of order $k$ relocates through the
unweighted sum at $V=k$, and the finishing map is one fixed alternating
pattern. A generic $k$-junta relocates through a weighted sum at $V=2^k$, and
the finishing map is an arbitrary table, with no shape a smooth learner could
exploit. Both families have $\Delta=\ONL-1$ once $V$ is large enough, so a
matched pair of equal order does not separate order from relocatability in the
abstract; it separates the capacity, and the kind of map, at which relocation
becomes available. Section~\ref{sec:junta} tests exactly this. When $k=n$ the
construction degenerates into memorisation ($2^n$ pieces), which is why it
offers no route on the multiplexer: there every input bit is relevant.
\textbf{Example.} For $k=2$, $g=x_1+3x_2$ takes the four distinct values
$-4,-2,+2,+4$ on the four assignments of $(x_1,x_2)$, so every Boolean
function of the two bits is a four-point read-out table on $g$. The
unweighted sum $x_1+x_2$ collapses the two mixed assignments to $0$: a
symmetric target such as XOR survives that, a table giving those two
assignments different values does not, and that difference is what the
weights $3^{j-1}$ erase.

Definition~\ref{def:delta} also extends to systems that evolve several
features. With $d$ programs $g_1,\dots,g_d$ and a learner $f:\R^d\to\R$ whose
capacity is counted in cells (leaves, for the trees we use), let
$\ONL^{\mathrm{dec}}_{\varepsilon,V,d}(y)$ be the least $r$ such that programs
of order at most $r$ each, combined by some $f$ of capacity $V$, reach $y$. For
parity the construction is explicit: split the $k$ relevant bits into $d$
blocks, let $g_j$ be the product over block $j$, of order $\lceil k/d\rceil$,
and let $f$ multiply the $d$ signs, which a tree can do with $2^d$ cells. The
reading is a trade, not a free lunch: as $d$ grows, the order demanded of each
feature falls like $\lceil k/d\rceil$, while the combination capacity demanded
of the learner grows like $2^d$. Section~\ref{sec:multi} tests both halves.

Four expectations follow, and Section~4 tests them in turn. With a linear learner,
a successful run must evolve a program whose order equals the target's. With a
sufficiently expressive learner, successful runs may use far lower order, and
should, because lower order programs are easier to find. The advantage of an
expressive learner should scale with $\Delta$ rather than with $\ONL(y)$, and
should vanish where $\Delta$ is near zero even if $\ONL(y)$ is large. And wherever
the learner absorbs part of the target's order it removes the selection pressure
on the program to supply it, so assimilation should be rarest exactly where the
learner does most.

\subsection{Experimental setup}

\paragraph{What we record.} Besides held out error and accuracy of the pipeline,
we record for the best individual of every generation of every run two further
quantities. One is the order of nonlinearity of the evolved program $g$, computed
exactly by Walsh-Hadamard transform on the Boolean targets and, on the continuous
ones, as the least polynomial degree reaching relative squared error $10^{-4}$ on a
held out probe sample. The other is the behaviour of $g$ \emph{before} the learner
is applied, which we call its \emph{innate} behaviour and which
Section~\ref{sec:overgen} needs. These two are the instruments of the paper: they
look at the programs rather than at how well the pipeline predicts.

\paragraph{Learners.} Linear regression, ridge, decision tree, random forest,
gradient boosting, and a multi-layer perceptron with one hidden layer of $50$ units
and early stopping, spanning a range of capacity $V$. Canonical GP, the first row
of Figure~\ref{fig:pipeline}, is the seventh condition.

\paragraph{Controls.} A limit of $30$ nodes applies identically to every
condition, so none has a larger budget than another. The Boolean primitive set is
$\{+,-,\times\}$ with the constant $1$, neutral between the two routes described
after Proposition~\ref{prop:parity}; the continuous set is
$\{+,-,\times,\div,\sin,\cos\}$ with an ephemeral constant. The learner in
Equation~\ref{eq:fit} is fitted under $3$-fold cross validation and fitness is the
mean held out fold error, for every learner including the linear ones. This matters: fitted and scored on the same rows, a tree or a forest
interpolates any one dimensional feature, training error falls to near zero for
almost every individual, and selection is left with nothing to act on.

\paragraph{Parameters and statistics.} Population $100$, $100$ generations,
tournament selection of size $3$, one point crossover $0.8$, uniform subtree
mutation $0.05$, elitism of one. Thirty independent runs per condition, $1890$
runs in the main grid, none failed; the dissociation and multi-feature
experiments of Sections~\ref{sec:junta} and~\ref{sec:multi} add $1{,}500$
more, the suite and control runs of Section~\ref{sec:bounds} a further
$2{,}310$, and the two measured datasets that close Section~\ref{sec:recipe}
another $420$, all under the same protocol. Outcomes on the Boolean targets are bimodal: a run
either classifies every held out row correctly or sits near chance, with
intermediate accuracies rare. We therefore treat success as a binary outcome and
test it with Fisher's exact test on the number of exact runs out of thirty, which
suits this outcome and is not weakened by the ties that blunt a rank test here.
Wilcoxon signed-rank tests on paired accuracies are reported where they add
information. Orders of nonlinearity are medians.

\section{Results}

\subsection{Contribution of each component in isolation}

We first separate the mechanism from the search. Every learner is handed the same
order-$1$ representation $s(\bx)=\sum_{i\le k}x_i$ guaranteed by
Proposition~\ref{prop:parity}, and only $f$ is fitted (Table~\ref{tab:probe}).
Given the same representation, a nonlinear learner is exact up to order seven and
a linear one recovers nothing above order one. The difference between learners is
therefore not a difference in raw power applied to a fixed problem; it is whether
the learner can absorb the order the representation does not carry. On $\MUX_6$,
which admits no such representation, every learner fails together.

\begin{table}[ht]\centering
\caption{Held out $R^2$ when the order-$1$ representation is supplied and only the
downstream learner is fitted. The linear column is at chance for every order above
one, which is Proposition~\ref{prop:affine} in numerical form.}
\label{tab:probe}
\begin{tabular}{lrrrrr}
\toprule
\textbf{Target} & \textbf{Order} & \textbf{Linear} & \textbf{Tree} & \textbf{Forest} & \textbf{Boosting} \\
\midrule
$\PAR_2$ & 2 & $-0.061$ & 1.000 & 1.000 & 1.000 \\
$\PAR_4$ & 4 & $-0.015$ & 1.000 & 1.000 & 1.000 \\
$\PAR_6$ & 6 & $-0.004$ & 1.000 & 1.000 & 1.000 \\
$\PAR_8$ & 8 & $-0.057$ & 0.948 & 0.948 & 0.948 \\
\midrule
$\MUX_6$ & 3 & 0.233 & 0.211 & 0.213 & 0.211 \\
\bottomrule
\end{tabular}
\end{table}

\begin{table}[ht]\centering
\caption{Held out accuracy of learners applied directly to the raw inputs, with no
evolved representation, mean over ten splits. Chance is $0.5$.}
\label{tab:direct}
\begin{tabular}{lrrrrrr}
\toprule
\textbf{Target} & \textbf{Order} & \textbf{Linear} & \textbf{Tree} & \textbf{Forest} & \textbf{Boosting} & \textbf{MLP} \\
\midrule
$\PAR_2$ & 2 & 0.377 & 0.849 & 1.000 & 0.994 & 1.000 \\
$\PAR_3$ & 3 & 0.378 & 0.840 & 0.996 & 0.891 & 1.000 \\
$\PAR_4$ & 4 & 0.383 & 0.630 & 0.712 & 0.227 & 0.995 \\
$\PAR_5$ & 5 & 0.388 & 0.469 & 0.273 & 0.238 & 0.819 \\
$\PAR_6$ & 6 & 0.390 & 0.369 & \textbf{0.112} & 0.221 & 0.716 \\
\bottomrule
\end{tabular}
\end{table}

Table~\ref{tab:direct} applies the same learners directly to the raw inputs.
Learning alone degrades with order, the random forest reaching $0.112$ at order
six, below chance, because its axis aligned splits are misled by a target
uncorrelated with every single input. Evolution alone fails from the opposite
direction: without a learner it is at chance from order four
(Table~\ref{tab:success}) and its median evolved representation has order $0$,
meaning it has collapsed to a constant (Table~\ref{tab:reloc}). The rest of this
section therefore concerns how the two divide the work, not whether combining them
helps.

\subsection{Success rate and order of the evolved programs}
\label{sec:reloc}

\begin{table}[ht]\centering
\caption{Number of runs out of thirty classifying every held out row correctly.
$p$ values are Fisher's exact test against LaSER-LR, the GP column included.
Dashes mark comparisons where both conditions are exact on all runs.}
\label{tab:success}
\begin{tabular}{lrrrrrrrrrrr}
\toprule
& & \multicolumn{2}{c}{\textbf{GP}} & \textbf{LR} & \textbf{Ridge}
& \multicolumn{2}{c}{\textbf{Tree}} & \multicolumn{2}{c}{\textbf{Forest}}
& \multicolumn{2}{c}{\textbf{Boosting}} \\
\cmidrule(lr){3-4}\cmidrule(lr){5-5}\cmidrule(lr){6-6}\cmidrule(lr){7-8}\cmidrule(lr){9-10}\cmidrule(lr){11-12}
\textbf{Target} & \textbf{Ord} & $n$ & $p$ & $n$ & $n$ & $n$ & $p$ & $n$ & $p$ & $n$ & $p$ \\
\midrule
$\PAR_1$ & 1 & 30 & --      & 30 & 30 & 30 & --      & 30 & --       & 30 & -- \\
$\PAR_2$ & 2 & 30 & --      & 30 & 30 & 30 & --      & 30 & --       & 30 & -- \\
$\PAR_3$ & 3 & 13 & $8\!\times\!10^{-6}$ & 29 & 28 & 30 & 1.0     & 30 & 1.0      & 30 & 1.0 \\
$\PAR_4$ & 4 &  2 & 0.0009  & 14 & 13 & 28 & 0.0001  & 30 & $2\!\times\!10^{-6}$ & 28 & 0.0001 \\
$\PAR_5$ & 5 &  0 & 0.005   &  8 &  7 & 24 & $7\!\times\!10^{-5}$ & 25 & $2\!\times\!10^{-5}$ & 21 & 0.002 \\
$\PAR_6$ & 6 &  0 & 0.024   &  6 &  7 & 22 & $7\!\times\!10^{-5}$ & 19 & 0.001    & 26 & $3\!\times\!10^{-7}$ \\
\midrule
$\MUX_6$ & 3 &  0 & 1.0     &  1 &  1 &  0 & 1.0     &  0 & 1.0      &  0 & 1.0 \\
\bottomrule
\end{tabular}
\end{table}

\begin{figure}[ht]\centering
\includegraphics[width=\linewidth]{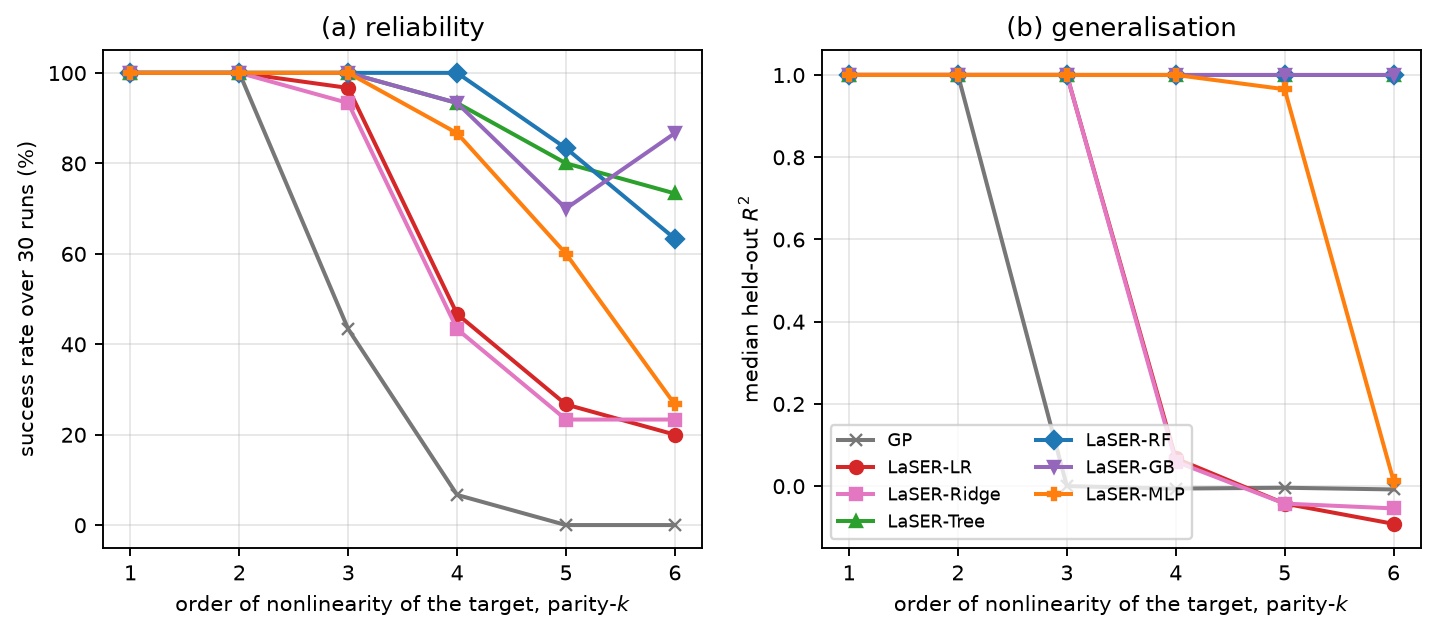}
\caption{Reliability and generalisation against the order of nonlinearity of the
target.}
\label{fig:perf}
\end{figure}

Three regimes are visible in Table~\ref{tab:success}. Up to order two every
condition is exact. From order three canonical GP falls away and LaSER-LR is
significantly better than it at every order from three to six
($p=8\times10^{-6}$, $0.0009$, $0.005$, $0.024$). From order four the nonlinear
learners separate from LaSER-LR, significantly at every order and for every
learner, with $p$ between $3\times10^{-7}$ and $0.002$; at order six the linear
variants are exact in six and seven runs of thirty against twenty-two, nineteen
and twenty-six for tree, forest and boosting. Wilcoxon tests on paired accuracies
agree, giving $p=0.32$ at order three for all three nonlinear learners and
$p\le0.00024$ at every order from four upwards.

Table~\ref{tab:reloc} is the measurement the paper turns on. It reports the order
of nonlinearity of the representation evolution produced, over the \emph{successful}
runs only: an unsuccessful run says nothing about the order that would have been
required, so an unconditional median would understate the constraint at high order
by averaging in representations that never worked.

\begin{table}[ht]\centering
\caption{Median order of nonlinearity of the evolved representation $g$, over
successful runs. A dash means no run succeeded. The linear variants lie on the
diagonal at every order without a single exception; the tree based learners hold
at two or three.}
\label{tab:reloc}
\begin{tabular}{lrrrrrrrr}
\toprule
\textbf{Target} & \textbf{Order} & \textbf{GP} & \textbf{LR} & \textbf{Ridge} & \textbf{Tree} & \textbf{Forest} & \textbf{Boosting} & \textbf{MLP} \\
\midrule
$\PAR_1$ & 1 & 1 & \textbf{1} & \textbf{1} & 1 & 1 & 3 & 1 \\
$\PAR_2$ & 2 & 2 & \textbf{2} & \textbf{2} & 1 & 1 & 2 & 2 \\
$\PAR_3$ & 3 & 0 & \textbf{3} & \textbf{3} & 2 & 2 & 3 & 3 \\
$\PAR_4$ & 4 & 0 & \textbf{4} & \textbf{4} & 2 & 2 & 2 & 4 \\
$\PAR_5$ & 5 & 0 & \textbf{5} & \textbf{5} & 2 & 3 & 3 & 4 \\
$\PAR_6$ & 6 & 0 & \textbf{6} & \textbf{6} & 3 & 3 & 3 & 6 \\
\midrule
$\MUX_6$ & 3 & 0 & 3 & 3 & -- & -- & -- & -- \\
\bottomrule
\end{tabular}
\end{table}

\begin{figure}[ht]\centering
\includegraphics[width=\linewidth]{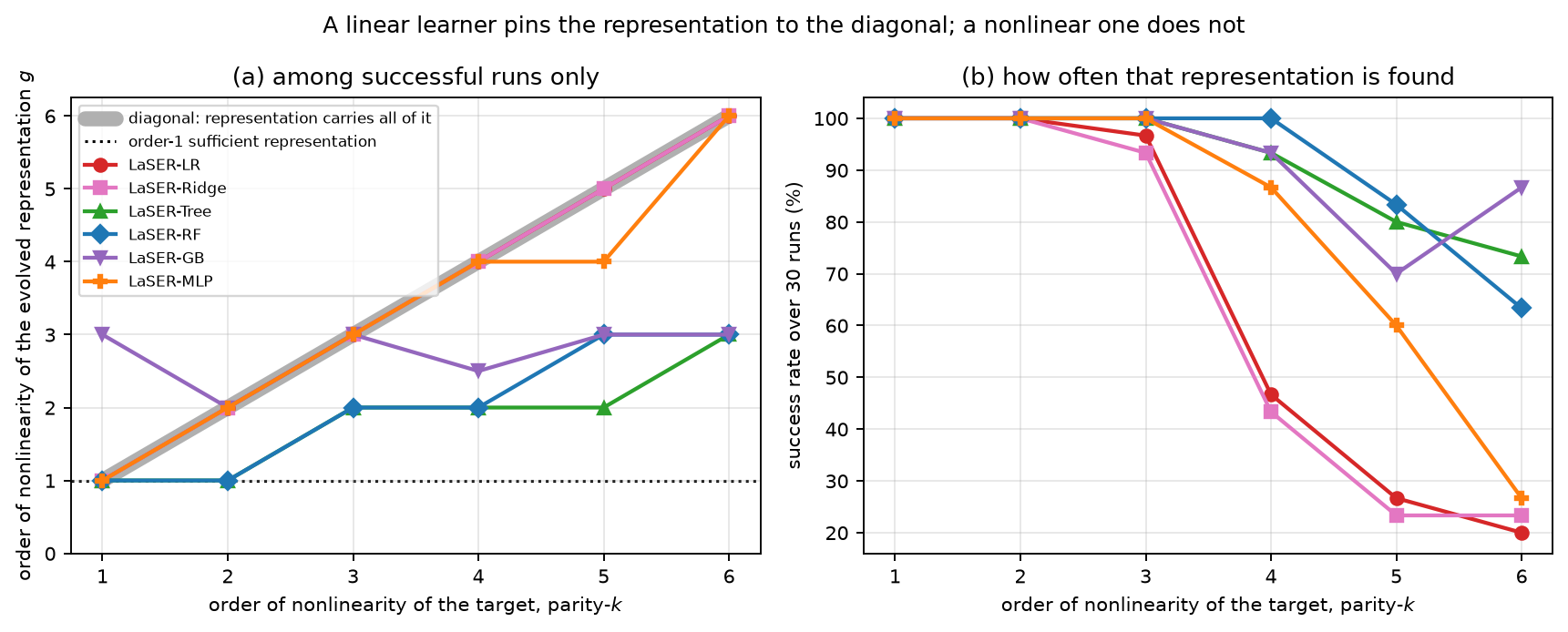}
\caption{(a) Median order of nonlinearity of the evolved representation
among successful runs, against the order of the target. The grey diagonal
band marks equality: a point on the band means evolution supplied exactly
the target's order, as an affine learner forces
(Proposition~\ref{prop:affine}); points below it mean part of the order was
relocated to the learner. (b) How often that representation is found
(successful runs out of thirty).}
\label{fig:reloc}
\end{figure}

As the target's order runs from one to six, the order of the representation
evolved by LaSER-LR and LaSER-Ridge runs $1,2,3,4,5,6$, with no exception in any
of the thirty runs per condition. This is Proposition~\ref{prop:affine} realised:
an affine map cannot raise the order of its input, so a run succeeds only if
evolution has supplied the target's order in full. The tree based learners succeed
with order two or three throughout. The difference between the two is the order of
nonlinearity that has moved into the learner, and panel (b) of
Figure~\ref{fig:reloc} is the price of not being able to move it.

Three further readings support the same conclusion. Taken unconditionally rather
than over successful runs, LaSER-LR and LaSER-Tree both give median order $3$ on
$\PAR_6$, yet their mean accuracies are $0.512$ and $0.925$: the representations
are of equal order, so the entire difference is what the learner does with them.
The budget also shows in tree size, median node counts at order six being $28$ for
LaSER-LR and $27$ for LaSER-Ridge against $19$ for the forest, and at order one
$7$ and $29$ against $3$; a linear learner forces evolution to spend both order
and structure. Finally, being nonlinear is necessary but not sufficient. LaSER-MLP
tracks the diagonal at $1,2,3,4,4,6$ and succeeds in eight runs of thirty at order
six, not significantly better than LaSER-LR ($p=0.76$), although it does separate
at orders four and five ($p=0.002$, $0.018$). With one hidden layer, early
stopping and roughly $180$ training rows it lacks the capacity to express the $k$
alternations that Proposition~\ref{prop:parity} requires, so in the one respect
that matters here it behaves like a linear learner. This is a quantitative
prediction of the analysis rather than an incidental observation.

\subsection{Equal order, different decomposition: random juntas}
\label{sec:junta}

The parity family raises the target's order and its relocatable part together,
so the results above cannot by themselves tell the two apart. The matched pair
that separates them holds the order fixed and changes only the decomposition.
We construct five random $4$-juntas on the same eight inputs: truth tables on
the first four bits, drawn once from a fixed pseudorandom stream and published
with the code, accepted only if their exact Fourier degree is $4$, the order of
$\PAR_4$, and they are not symmetric in their four arguments, so
Proposition~\ref{prop:symmetric} offers no route and
Proposition~\ref{prop:junta} is the only decomposition available. The other
four bits are irrelevant but visible, as for parity. All seven conditions ran
on each junta, thirty runs per cell.

\begin{table}[ht]\centering
\caption{Matched-pair dissociation: $\PAR_4$ against five random $4$-juntas,
all of exact order $4$. Each cell gives the number of runs out of thirty
classifying every held out row correctly, with the median order of the evolved
$g$ among those runs in parentheses; a dash means no run succeeded. The pooled
row sums the junta rows (out of $150$).}
\label{tab:junta}
\begin{tabular}{lrrrrrrr}
\toprule
\textbf{Target} & \textbf{GP} & \textbf{LR} & \textbf{Ridge} & \textbf{Tree} & \textbf{Forest} & \textbf{Boosting} & \textbf{MLP} \\
\midrule
$\PAR_4$ & 2 (4) & 14 (4) & 13 (4) & 28 (2) & 30 (2) & 28 (2.5) & 26 (4) \\
Junta 1 & 0 (--) & 19 (3) & 13 (3) & 30 (2) & 30 (2) & 30 (2) & 15 (3) \\
Junta 2 & 0 (--) & 1 (3) & 4 (3) & 28 (2) & 25 (2) & 22 (2.5) & 9 (3) \\
Junta 3 & 0 (--) & 9 (2) & 14 (2.5) & 28 (1.5) & 29 (2) & 29 (2) & 12 (2) \\
Junta 4 & 0 (--) & 18 (3) & 14 (3) & 29 (3) & 30 (3) & 30 (3) & 22 (3) \\
Junta 5 & 0 (--) & 1 (4) & 3 (3) & 28 (1.5) & 29 (2) & 29 (2) & 15 (3) \\
\midrule
Pooled & 0 & 48 & 48 & 143 & 143 & 140 & 73 \\
\bottomrule
\end{tabular}
\end{table}

Table~\ref{tab:junta} shows the outcome, and it is sharper than the
expectation the pair was designed around. The tree based learners keep their
advantage undiminished: pooled over the five juntas they are exact in $143$,
$143$ and $140$ runs of $150$ against $48$ for LaSER-LR (Fisher's exact test,
$p\le2\times10^{-30}$ pooled and $p\le1.1\times10^{-3}$ on every junta
separately), their success on the juntas is indistinguishable from their
success on $\PAR_4$ ($p=0.65$ for the tree), and the median order of their
evolved programs stays at two. Relocation, for them, works as well on an
arbitrary table as on parity. What changes is who else can follow. The MLP
drops from $26$ of $30$ on $\PAR_4$ to $73$ of $150$ pooled
($p=9.8\times10^{-5}$): it can learn the one alternating map of a sum but not
an arbitrary table over a projected feature, and where it does succeed it needs
order $3$ from evolution rather than the $4$ it tolerated on parity's single
monomial. The affine learners become strongly instance dependent, between $1$
and $19$ successes of $30$ across the five tables against a stable $14$ on
$\PAR_4$: with a dense degree-$4$ spectrum in place of one monomial, success
rests on evolving most of the table, and some tables happen to be reachable at
sign agreement from order $3$ while others are not. Canonical GP solves no
junta run at all ($0$ of $150$).

The mechanism is visible directly on the evolved objects. In every one of the
$143$ successful tree runs, the level sets of $g$ are label pure: inputs
sharing a value of $g$ share the target's label, so the program has built an
injective-enough projection and the learner has tabulated it, which is the
construction of Proposition~\ref{prop:junta} found by search rather than by
hand. The price appears in the size of the table: the median number of
distinct values of $g$ across the full probe is $12$ on the juntas against $5$
on $\PAR_4$, where the five values are the levels of the parity sum. Same
order, same relocated order, but a finishing map that has gone from five
structured pieces to a dozen arbitrary ones; the learners that fail on the
juntas are exactly the ones for which that distinction matters. $\Delta$ as
Definition~\ref{def:delta} states it is indexed by a capacity, and the pair
shows the index doing real work: what predicts the payoff of a learner is not
order, and not relocatability in the abstract, but whether the learner's
capacity and bias cover the finishing map the target actually needs.

\subsection{Relocation with several features}
\label{sec:multi}

The second generalisation replaces the single evolved program by $d$ of them.
The genome is a list of $d$ trees; each tree is limited to
$\lceil30/d\rceil$ nodes so that the total budget matches the single-feature
grid; crossover and mutation act on one tree at a time; the learner is fitted
on the $d$ semantic columns under the same three fold cross validation, and
the order of every feature is computed exactly as before. To span combination
capacity we add two depth limited trees to the ladder: tree$_2$ and tree$_3$,
with at most $4$ and $8$ leaves. $\PAR_4$ runs with LR, tree$_2$ and the
unbounded tree at $d\in\{1,2\}$, $\PAR_6$ with tree$_2$, tree$_3$ and the
unbounded tree at $d\in\{1,2,3\}$, thirty runs per cell, $450$ runs.

\begin{table}[ht]\centering
\caption{Several features. Each cell gives successful runs out of thirty and,
in parentheses, the median order of nonlinearity per evolved feature among
successful runs. tree$_2$ and tree$_3$ are depth limited to $4$ and $8$
leaves.}
\label{tab:multi}
\begin{tabular}{llrrr}
\toprule
\textbf{Target} & \textbf{Learner} & $d=1$ & $d=2$ & $d=3$ \\
\midrule
$\PAR_4$ & LR & 25 (4) & 9 (3) & \\
 & tree$_2$ & 30 (2) & 30 (2) & \\
 & tree & 30 (2) & 30 (2) & \\
\midrule
$\PAR_6$ & tree$_2$ & 17 (4) & 14 (3) & 2 (2.5) \\
 & tree$_3$ & 25 (3) & 21 (2) & 18 (2) \\
 & tree & 28 (3) & 24 (2) & 26 (2) \\
\bottomrule
\end{tabular}
\end{table}

Both halves of the trade predicted in Section~\ref{sec:relocate} appear in
Table~\ref{tab:multi}. The order demanded of each feature falls as $d$ rises:
for the unbounded tree on $\PAR_6$ the per-feature median runs $3,2,2$ across
$d=1,2,3$ (Mann-Whitney on per-feature orders among successful runs, $d=1$
against $d=3$, $p=2.6\times10^{-10}$; for tree$_3$, $p=5.7\times10^{-8}$),
and on $\PAR_4$ at $d=2$ the two features settle at order two each, the block
decomposition found by search. The combination burden rises in step: composing
$d$ block parities needs $2^d$ cells, so tree$_2$, with four leaves, holds at
$d\le2$ and collapses to $2$ of $30$ at $d=3$ ($p=5.4\times10^{-5}$ against
its own $d=1$), while tree$_3$, with eight, holds through $d=3$. More features
make each program's job easier and the learner's job harder, at the rates the
construction says.

The affine case extends Proposition~\ref{prop:affine} verbatim: a linear
combination of features adds no order, so the pipeline succeeds only if some
single feature carries the target's full order. All nine successful LR runs at
$d=2$ contain a feature of order exactly $4$, with no exception, and success
falls from $25$ of $30$ at $d=1$ to $9$ of $30$ at $d=2$
($p=6.3\times10^{-5}$): splitting the budget between two features only makes
the one order-$4$ feature harder to find. Under an affine learner, extra
features are not a resource but a tax, which is the multi-feature form of the
boundary the proposition draws.

\subsection{Multiplexer control and continuous targets}
\label{sec:bounds}

$\MUX_6$ has order $3$, lower than $\PAR_4$ through $\PAR_6$, and is not
symmetric. No condition solves it: the best is one exact run of thirty, by
LaSER-LR. Fisher's exact test finds no difference between LaSER-LR and any
nonlinear learner ($p=1.0$ in every cell of Table~\ref{tab:success}), and Wilcoxon
tests on accuracy agree ($p=0.58$, $0.97$, $0.63$, $0.11$ for tree, forest,
boosting and MLP). On a squared error metric the same comparison gives marginal
values of $0.03$ and $0.05$ for the tree based learners, so we do not claim exact
equality, only that the large and consistent separation present from order four
upwards is absent here. A target can therefore carry substantial order of
nonlinearity and still reward an expressive learner not at all. What predicts the
reward is how much of the target's order can be relocated, not how much order it
has.

The control replicates at scale. The $11$-multiplexer, in which three address
bits select one of eight data bits, has exact order $4$ on its full table of
$2^{11}$ rows, below $\PAR_5$ and $\PAR_6$. No condition solves a single run
of thirty, and the differences that do exist run in the linear learner's
favour: median held out accuracy is $0.667$ for LaSER-LR against $0.645$,
$0.652$ and $0.627$ for tree, forest and boosting (Wilcoxon $p=0.03$, $0.017$,
$1.2\times10^{-4}$), with the MLP indistinguishable ($p=0.58$) and canonical
GP at chance. The reason is the one Proposition~\ref{prop:junta} names: every
one of the eleven input bits is relevant, so the table route degenerates into
memorisation, and no low order sufficient statistic exists for a univariate
learner to finish.

\begin{table}[ht]\centering
\caption{Continuous targets, thirty runs each. Orders of $g$ are medians
over successful runs. $p$ values are Wilcoxon signed-rank tests against
LaSER-LR on test MSE.}
\label{tab:cont}
\begin{tabular}{llrrrl}
\toprule
\textbf{Target} & \textbf{Method} & \textbf{Test MSE} & \textbf{Succ.\ \%} & \textbf{Order of $g$} & \textbf{vs LaSER-LR} \\
\midrule
Nguyen-11 & GP         & 0.00000 & 100 & 2 & \\
(order 2) & LaSER-LR   & 0.00000 & 100 & 2 & \\
          & LaSER-Tree & 0.00001 & 100 & 2 & $p<10^{-8}$, LR better \\
          & LaSER-RF   & 0.00001 & 100 & 6 & $p<10^{-8}$, LR better \\
          & LaSER-GB   & 0.00003 & 100 & 3 & $p<10^{-8}$, LR better \\
          & LaSER-MLP  & 0.00003 & 100 & 2 & $p<10^{-8}$, LR better \\
\midrule
Nguyen-8  & GP         & 0.00062 &  67 & 3 & \\
(order 8) & LaSER-LR   & 0.00003 & 100 & 5 & \\
          & LaSER-Tree & 0.00000 & 100 & 4 & $p<10^{-8}$, Tree better \\
          & LaSER-RF   & 0.00000 & 100 & 6 & $p<10^{-8}$, RF better \\
          & LaSER-GB   & 0.00001 & 100 & 4 & $p=1\!\times\!10^{-6}$, GB better \\
          & LaSER-MLP  & 0.00002 & 100 & 3 & $p=0.096$ \\
\bottomrule
\end{tabular}
\end{table}

Table~\ref{tab:cont} shows the predicted ordering outside the instrument: on the
low order target the linear learner is significantly better than every nonlinear
one, and on the high order target every tree based learner is significantly better
than the linear one. The same table shows why the Boolean instrument was needed.
With a tolerance, Proposition~\ref{prop:affine} binds only softly: LaSER-LR
reaches order $5$ on a target of order $8$ and still attains $R^2\ge0.99$ in all
thirty runs, because a smooth target can be approximated below its exact order.
Every condition succeeds here and the effect survives only in the magnitude of the
error. At $\varepsilon=0$ on the cube the constraint binds exactly, which is why
the separation in Section~\ref{sec:reloc} is visible at all.

The same ordering holds across the rest of the suite.
Table~\ref{tab:ngall} runs every condition on all twelve Nguyen targets,
$2{,}310$ further runs. The tree and forest learners are significantly better
than LaSER-LR on eight of the twelve targets, and the size of the linear
learner's disadvantage grows with the target's estimated order: the Spearman
correlation between the order and the log ratio of median errors, LaSER-LR
over LaSER-Tree, is $0.73$ ($p=0.007$) across the suite. The reversals sit
where the analysis puts them. On the two targets the affine pipeline solves to
machine precision, Nguyen-10 and Nguyen-11, of orders $3$ and $2$, every
nonlinear learner is equal or significantly worse; on Nguyen-9, where LaSER-LR
again reaches a very low error, no tree based learner separates from it. The
one suite member essentially unsolved by every condition, Nguyen-12, orders
the learners by neither reading (the MLP is best there, $p=0.03$), and with at
most nine successful runs of thirty in any condition we do not read direction
into it.

\begin{table}[ht]\centering
\caption{The full Nguyen suite, thirty runs per cell. Entries are median held
out test MSE; $\approx0$ marks medians below $10^{-20}$. Bold: significantly
better than LaSER-LR; italic: significantly worse (Wilcoxon signed-rank on
per-seed MSE, $p<0.05$). Orders are estimated at relative
$\varepsilon=10^{-4}$, as throughout.}
\label{tab:ngall}
\setlength{\tabcolsep}{4pt}
\begin{tabular}{lrrrrrrr}
\toprule
\textbf{Target} & \textbf{Ord} & \textbf{GP} & \textbf{LR} & \textbf{Tree} & \textbf{Forest} & \textbf{Boosting} & \textbf{MLP} \\
\midrule
Nguyen-1 & 3 & $\approx0$ & $1.5\!\times\!10^{-4}$ & \textbf{$2.3\!\times\!10^{-5}$} & \textbf{$1.8\!\times\!10^{-5}$} & $1.1\!\times\!10^{-4}$ & $1.3\!\times\!10^{-4}$ \\
Nguyen-2 & 4 & $2.1\!\times\!10^{-4}$ & $4.0\!\times\!10^{-4}$ & \textbf{$4.0\!\times\!10^{-5}$} & \textbf{$3.1\!\times\!10^{-5}$} & \textbf{$1.6\!\times\!10^{-4}$} & \textbf{$2.0\!\times\!10^{-4}$} \\
Nguyen-3 & 5 & $2.3\!\times\!10^{-3}$ & $6.1\!\times\!10^{-4}$ & \textbf{$6.7\!\times\!10^{-5}$} & \textbf{$5.3\!\times\!10^{-5}$} & \textbf{$2.2\!\times\!10^{-4}$} & \textbf{$3.1\!\times\!10^{-4}$} \\
Nguyen-4 & 6 & $2.5\!\times\!10^{-3}$ & $8.4\!\times\!10^{-4}$ & \textbf{$1.0\!\times\!10^{-4}$} & \textbf{$8.0\!\times\!10^{-5}$} & \textbf{$3.3\!\times\!10^{-4}$} & $4.7\!\times\!10^{-4}$ \\
Nguyen-5 & 4 & $4.2\!\times\!10^{-4}$ & $4.4\!\times\!10^{-7}$ & \textbf{$2.8\!\times\!10^{-7}$} & \textbf{$2.0\!\times\!10^{-7}$} & \emph{$1.9\!\times\!10^{-6}$} & \emph{$2.6\!\times\!10^{-6}$} \\
Nguyen-6 & 5 & $9.0\!\times\!10^{-4}$ & $1.8\!\times\!10^{-4}$ & \textbf{$9.8\!\times\!10^{-6}$} & \textbf{$7.2\!\times\!10^{-6}$} & \textbf{$6.6\!\times\!10^{-5}$} & $2.1\!\times\!10^{-4}$ \\
Nguyen-7 & 3 & $3.5\!\times\!10^{-4}$ & $2.0\!\times\!10^{-5}$ & \textbf{$7.3\!\times\!10^{-6}$} & \textbf{$5.3\!\times\!10^{-6}$} & \emph{$6.4\!\times\!10^{-5}$} & $1.9\!\times\!10^{-5}$ \\
Nguyen-8 & 8 & $4.5\!\times\!10^{-4}$ & $2.1\!\times\!10^{-5}$ & \textbf{$1.6\!\times\!10^{-6}$} & \textbf{$1.2\!\times\!10^{-6}$} & \textbf{$7.2\!\times\!10^{-6}$} & $1.3\!\times\!10^{-5}$ \\
Nguyen-9 & 4 & $2.1\!\times\!10^{-3}$ & $4.9\!\times\!10^{-5}$ & $2.6\!\times\!10^{-4}$ & $2.4\!\times\!10^{-4}$ & $2.0\!\times\!10^{-4}$ & \emph{$4.4\!\times\!10^{-4}$} \\
Nguyen-10 & 3 & $1.1\!\times\!10^{-3}$ & $\approx0$ & $1.4\!\times\!10^{-5}$ & $1.0\!\times\!10^{-5}$ & \emph{$9.7\!\times\!10^{-5}$} & \emph{$2.3\!\times\!10^{-4}$} \\
Nguyen-11 & 2 & $\approx0$ & $\approx0$ & \emph{$9.1\!\times\!10^{-6}$} & \emph{$9.1\!\times\!10^{-6}$} & \emph{$2.9\!\times\!10^{-5}$} & \emph{$2.3\!\times\!10^{-5}$} \\
Nguyen-12 & 4 & $3.5\!\times\!10^{-2}$ & $1.3\!\times\!10^{-2}$ & \emph{$2.1\!\times\!10^{-2}$} & $1.5\!\times\!10^{-2}$ & $1.0\!\times\!10^{-2}$ & \textbf{$8.9\!\times\!10^{-3}$} \\
\bottomrule
\end{tabular}
\end{table}

\subsection{Choosing a learner when the order is unknown}
\label{sec:recipe}

In practice the order of nonlinearity of given data is not known in advance.
Three readings used throughout this paper are available without knowing it,
two before any evolution is run and one during it, and together they amount
to a procedure. We state it, walk it through three synthetic targets treated
as unknown data, and close with two measured datasets where no ground truth
exists at all.

First, estimate the order directly. Definition~\ref{def:onl} is itself an
estimator: fit polynomials of increasing degree by least squares and read off
where the relative error crosses the tolerance, exactly as
Figure~\ref{fig:domain}(b) does. On Nguyen-11 the error falls from $1.0$ at
degree one to $5\times10^{-31}$ at degree two, so the reading is order $2$;
on Nguyen-8 it crosses $10^{-4}$ at degree eight. Second, fit the learner
ladder of Table~\ref{tab:direct} directly to the raw inputs. If the capable
end of the ladder already succeeds there, the data admit a compact
multivariate model and the question of what to attach to a single evolved
feature does not arise; if the whole ladder fails, an evolved representation
is doing necessary work. Third, during any run, compare the order of the
evolved program, free to compute, with the estimate from the first reading:
a successful run whose program sits well below the estimate is relying on
the learner, and one whose program sits at the estimate is paying for a
learner that contributes nothing.

The three targets separate cleanly under this procedure. Nguyen-11 reads
order $2$: low, so the program can carry it all and the cheapest learner is
the right one, and in Table~\ref{tab:ngall} LaSER-LR is indeed significantly
better than every nonlinear alternative. Nguyen-8 reads order $8$: high, so
an expressive learner should pay, and it does ($p\le10^{-6}$ for all three
tree based learners); the third reading confirms the division of labour,
successful tree runs carrying programs of median order $4$ against the
target's $8$, while LaSER-LR must reach $5$ and still trails. The
$11$-multiplexer is the case the first reading alone would miscall: its
order, $4$, is modest, yet the pipeline fails at every rung of the ladder.
The second reading catches it: a plain decision tree fitted to the raw
eleven bits reaches accuracy $0.999$ (forest $1.000$, linear $0.614$, same
ten-split protocol as Table~\ref{tab:direct}), so the data are easy for a
multivariate model and the obstruction is the one dimensional channel
through a single evolved feature, which is what Section~\ref{sec:multi}
relaxes. The procedure does not compute $\Delta$, which remains open for
arbitrary targets; it brackets it from both sides with quantities that cost
a few seconds each, and on every target in this paper the bracket picks the
right learner.

Every target so far was synthetic, chosen so that the order could be
verified. As a closing check we run the same procedure on two measured
datasets from the UCI repository~\cite{kelly2023uci}, where the order is
genuinely unknown and the data carry real noise: airfoil self-noise
($1{,}503$ rows, five inputs, wind tunnel measurements of sound
pressure~\cite{brooks1989airfoil}) and yacht hydrodynamics ($308$ rows, six
inputs, towing tank measurements of residuary
resistance~\cite{gerritsma1981yacht}). Inputs are standardised per feature,
an affine map that by Proposition~\ref{prop:affine} moves no order; the
target is untouched. The first reading now meets its practical limits and
still reads. On yacht the polynomial error falls steeply through the degrees
the rows can identify, $0.34$, $0.089$, $0.013$, $0.0014$ at degrees one to
four, and the design matrix runs out of rows before the fall stops: high
order, little noise. On airfoil the error bottoms out at $0.19$ at degree
four and rises after, a floor that is part noise and part structure no
polynomial of identifiable degree reaches. Both readings say expressive
learner, and the second reading agrees: fitted directly to the raw inputs,
the forest reaches $R^2=0.93$ on airfoil against $0.51$ for the linear
model, and boosting $0.997$ on yacht against $0.63$.

Table~\ref{tab:real} shows the pipeline, thirty runs per condition as
throughout. The reading is correct on both datasets: every tree based
learner is significantly better than LaSER-LR ($p\le5.7\times10^{-7}$ on
airfoil, $p\le2.1\times10^{-5}$ on yacht, Wilcoxon on per-seed test MSE;
ridge is indistinguishable from LaSER-LR on both), and the programs divide
the labour as on the synthetic targets: under the forest the median evolved
order is $2.5$ on airfoil and $3.5$ on yacht, while under LaSER-LR the
median program is already past the highest degree the probe rows can
identify (five and four respectively). The two caveats of the synthetic
study reappear unchanged. On airfoil the pipeline forest ($R^2=0.70$)
trails the direct forest ($0.93$): five interacting inputs do not fit
through a one dimensional channel, the multiplexer's obstruction, and the
multi-feature arrangement of Section~\ref{sec:multi} is the remedy the
procedure points to. On yacht the relocation runs the other way: LaSER-LR
reaches $R^2=0.98$ where the direct linear model stops at $0.63$, the
evolved program carrying the order the learner cannot, so one evolved
feature plus any rung of the ladder solves the problem and the procedure's
advice reduces to taking the cheapest rung that holds.

\begin{table}[ht]\centering
\caption{Two measured datasets, thirty runs per cell. Entries are median
held out test MSE in the target's original units, so the rows are not
comparable with each other; bold and italic as in Table~\ref{tab:ngall}
(significantly better or worse than LaSER-LR, Wilcoxon signed-rank on
per-seed MSE, $p<0.05$).}
\label{tab:real}
\setlength{\tabcolsep}{5pt}
\begin{tabular}{lrrrrrr}
\toprule
\textbf{Dataset} & \textbf{GP} & \textbf{LR} & \textbf{Tree} & \textbf{Forest} & \textbf{Boosting} & \textbf{MLP} \\
\midrule
Airfoil & $5604$ & $20.6$ & \textbf{16.6} & \textbf{14.0} & \textbf{16.2} & \emph{382} \\
Yacht & $20.7$ & $3.64$ & \textbf{2.24} & \textbf{1.83} & \textbf{2.16} & $4.03$ \\
\bottomrule
\end{tabular}
\end{table}

\subsection{Innate and learned behaviour across generations}
\label{sec:overgen}

The learner is refitted every generation and never inherited, so this is lifetime
learning, and the innate behaviour recorded at every generation lets us ask
whether evolution comes to do the job alone. Two hypotheses are separable from the
trajectories: under assimilation the innate behaviour rises and the gap to the
learned behaviour closes; under what we will call evolving learnability the
learned behaviour rises while the innate behaviour stays flat or falls, so the gap
widens.

\begin{figure}[ht]\centering
\includegraphics[width=\linewidth]{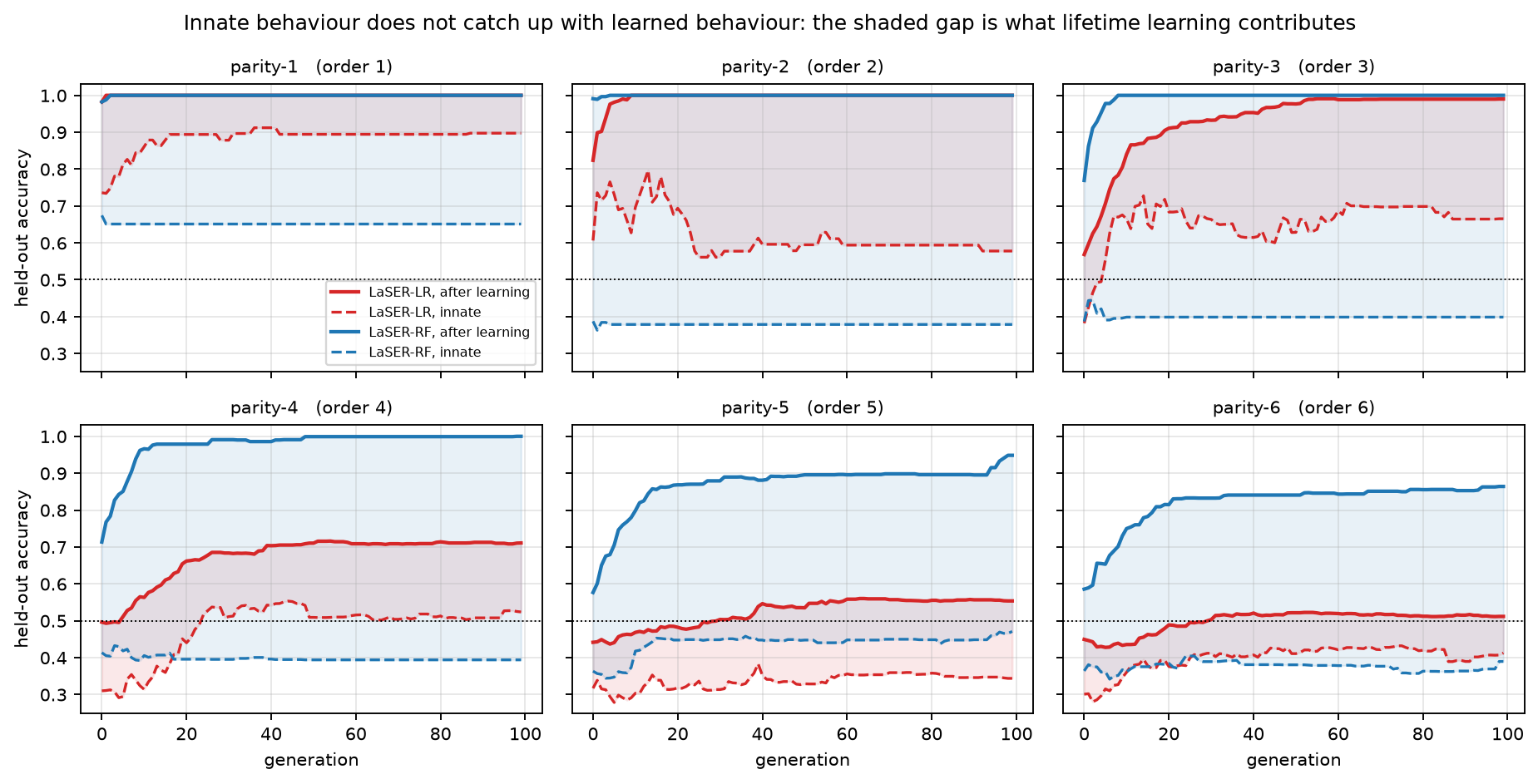}
\caption{Innate (dashed) and learned (solid) accuracy of the best individual, mean
over thirty runs; the shaded area is what lifetime learning contributes.}
\label{fig:baldwin}
\end{figure}

\begin{figure}[ht]\centering
\includegraphics[width=\linewidth]{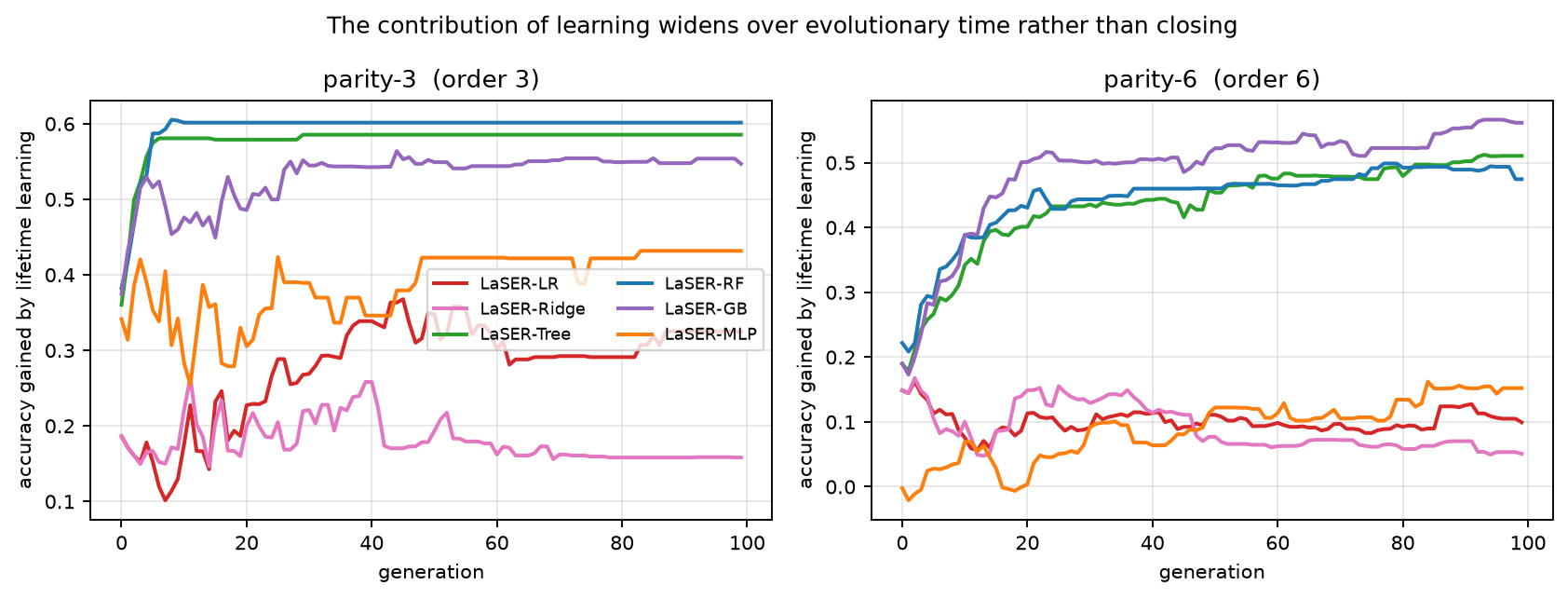}
\caption{Accuracy contributed by lifetime learning, over generations.}
\label{fig:gap}
\end{figure}

Across the sixty-three problem and learner combinations of both domains,
assimilation occurs in four, three of which are canonical GP, for which success
and innateness are by definition the same event. In forty-four the gap widens
instead. On $\PAR_4$ under a random forest the innate accuracy of the best
individual \emph{falls} from $0.713$ at the first generation to $0.394$ at the
last while its learned accuracy rises from $0.713$ to $1.000$: evolution moves
away from representations that predict the target on their own, because the
learner has made that property unnecessary.

\begin{table}[ht]\centering
\caption{Percentage of runs reaching perfect innateness, in which the evolved
representation alone attains the target and the learner is redundant.}
\label{tab:innate}
\begin{tabular}{lrrrrrrr}
\toprule
\textbf{Target} & \textbf{GP} & \textbf{LR} & \textbf{Ridge} & \textbf{Tree} & \textbf{Forest} & \textbf{Boosting} & \textbf{MLP} \\
\midrule
$\PAR_1$ & 100.0 & 3.3 & 0.0 & 26.7 & 26.7 & 0.0 & 0.0 \\
$\PAR_2$ & 100.0 & 3.3 & 0.0 & 13.3 & 13.3 & 0.0 & 3.3 \\
$\PAR_3$ &  43.3 & 6.7 & 0.0 & 10.0 & 13.3 & 0.0 & 3.3 \\
$\PAR_4$ &   6.7 & 10.0 & 0.0 & 3.3 & 3.3 & 0.0 & 6.7 \\
$\PAR_5$ &   0.0 & 0.0 & 0.0 & 0.0 & 0.0 & 0.0 & 0.0 \\
$\PAR_6$ &   0.0 & 3.3 & 0.0 & 0.0 & 0.0 & 0.0 & 3.3 \\
\midrule
Nguyen-11 & 100.0 & 10.0 & 0.0 & 53.3 & 10.0 & 10.0 & 0.0 \\
Nguyen-8  &   0.0 & 0.0 & 0.0 & 0.0 & 0.0 & 0.0 & 0.0 \\
\bottomrule
\end{tabular}
\end{table}

\begin{figure}[ht]\centering
\includegraphics[width=\linewidth]{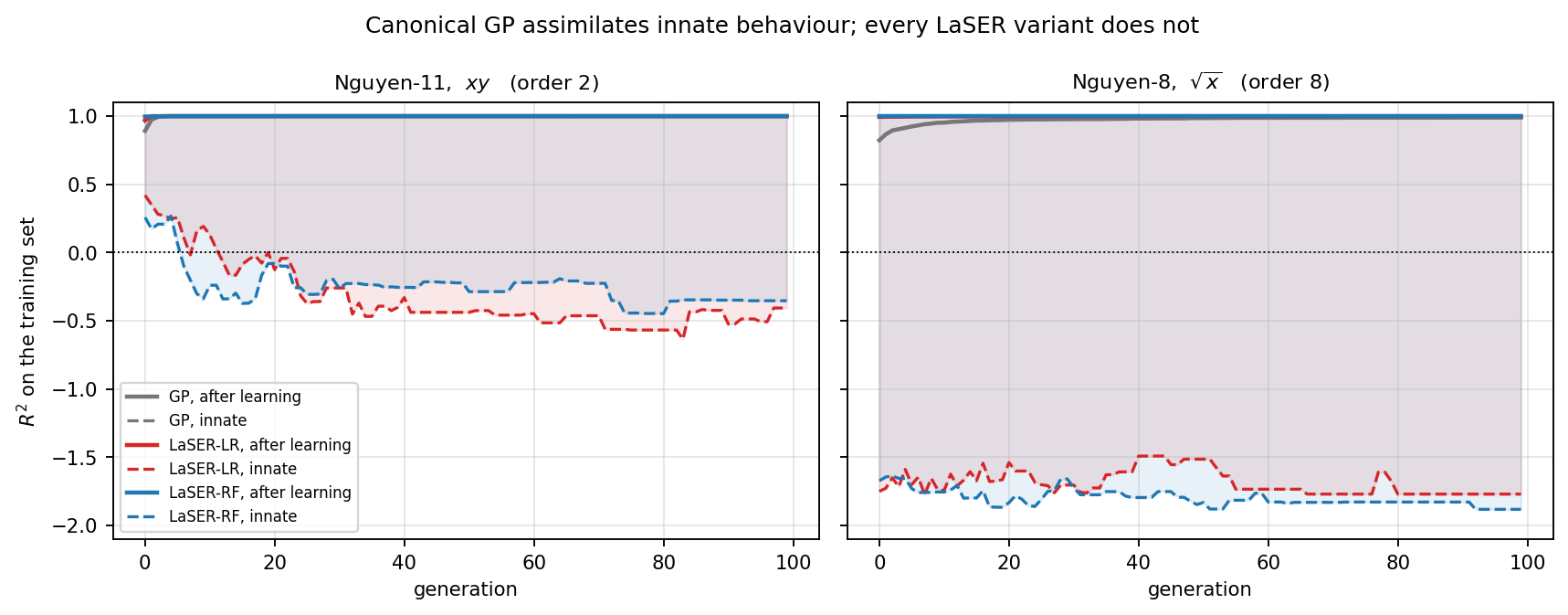}
\caption{Continuous targets: canonical GP assimilates, every learner condition
does not. $R^2$ floored at $-2$ for legibility.}
\label{fig:baldreg}
\end{figure}

Table~\ref{tab:innate} gives the endpoint. Perfect innateness occurs only where
the target's order is lowest and never at all above order four, and
Figure~\ref{fig:baldreg} shows the same on the continuous targets, where canonical
GP assimilates, the $R^2$ of its innate representation rising from $0.892$ to
$1.000$ on Nguyen-11, while every learner condition moves away from innate
competence. One result cuts against the simplest reading: tree and forest
assimilate more often than LaSER-LR at the lowest order, $26.7\%$ against $3.3\%$
on $\PAR_1$, which capacity alone does not predict. The explanation is that
shielding is governed by the size of the set of representations the learner maps
onto the target, not by capacity. For an affine learner that set is a two
parameter family, so every $g$ affinely related to the target already attains
optimal fitness and there is no pressure to land on the target exactly. Capacity
and preimage size coincide on continuous regression problems, which is why the
distinction has not been visible before.

A measurement caveat applies to the continuous half. On Boolean targets the innate
measure is sign based and so invariant to rescaling of $g$. On continuous targets
it is not: $R^2$ of $g$ against the target penalises scale differences that an
affine learner removes for free, so the continuous numbers measure \emph{strict}
innateness, meaning $g$ equalling the target, not the structure $g$ captures. The
Boolean measurements are the primary evidence here for that reason.

\section{Discussion}

A word on the reach of the measure itself. Definition~\ref{def:onl} is not
Boolean: it applies on any real domain, and the continuous experiments
already use it with the relative tolerance convention. What the Boolean
domain adds is exactness, through Proposition~\ref{prop:fourier}. Away from
it, three practical choices appear: the tolerance, the basis (polynomials
here; any graded basis supports the same definition), and enough data to
identify the fit. Categorical inputs enter through any $\pm1$ encoding, and
mixed domains through products of the two cases. When none of this is
available, order can still be probed behaviourally, by fitting a ladder of
learners of increasing capacity directly to the data and watching where the
fit stops improving; Section~\ref{sec:recipe} turns this into a stated procedure.

The usual account of a genetic programming system paired with a downstream learner
is additive: evolution gets part of the way and the learner absorbs the residual
error. The measurements here do not support that account. The order of
nonlinearity that evolution must supply changes from six to three on the same
target, with the size limit, primitive set, and every search parameter held
fixed and only the learner changed. The total structural complexity of the
problem is not reduced by the learner; it is paid for somewhere else. What a
learner buys is a transfer of budget, and the measurement of
Section~\ref{sec:reloc} makes that transfer visible directly on the evolved
objects, instead of inferring it from error.

That gives a criterion where previously there was a benchmark result. It is not
that expressive learners suit hard problems. $\MUX_6$ is of lower order than the
parity targets on which the separation is largest, and on it no learner helps at
all; Nguyen-11 is of low order and there a linear learner is significantly better
than every alternative. What decides the choice is whether the target admits a low
order representation from which a bounded learner can recover it, which is the
quantity $\Delta$ of Definition~\ref{def:delta}. Difficulty and the appropriate
downstream model are orthogonal, and an experimenter reading difficulty off a
benchmark's reputation will choose wrongly in both directions.

The second finding is uncomfortable for the paradigm itself, not for any
particular system in it. Genetic programming is often chosen because its models
are symbolic and stand on their own. Attaching a learner to improve accuracy
systematically removes the selection pressure that would make them do so, and the
effect is not incidental: on $\PAR_4$ under a random forest, evolution actively
moves away from representations that predicted the target unaided. The more
capable the learner, the smaller the share of the model that resides in the
readable part. This tension is invisible when the downstream stage is fixed to be
affine, because then there is nothing for the learner to absorb, and it applies
equally to multidimensional systems that pair evolved features with a
deterministic classifier or regressor~\cite{munoz2015m3gp,lacava2019feat}.

Two results previously reported for the single-feature form of this
arrangement~\cite{le2026laser} are recovered here and read differently. The
pattern already noted, a linear learner dominating one benchmark suite while
nonlinear learners take over on a harder one, is reproduced here, but the two
suites differ in $\Delta$, not in difficulty; structure was never the binding
constraint in the standard suites. That the evolved representation
sometimes attains the target unaided is also reproduced, and Table~\ref{tab:innate}
locates those cases precisely: they are the low order corner of the space, they
never occur above order four, and the prevailing direction across the full range
is the opposite one.

\paragraph{Limitations.} The parity family raises the target's order and the
relocatable part of it together, since parity is symmetric at every $k$ and
hence $\Delta=k-1$; Section~\ref{sec:junta} removes that confound on the
Boolean side with matched pairs of equal order, but its continuous counterpart,
a matched pair such as $(w^\top\bx)^6$ against an additive degree-$6$
polynomial, remains untested. The multi-feature experiments stop at $d=3$ and
at parity targets, so the trade they exhibit has not been traced on targets
whose block structure is less clean. The Boolean experiments use
full truth tables on eight inputs, which keeps the Fourier analysis exact but
limits scale, since larger $n$ would require sampling and reintroduce estimation
into the quantity the study is built on. We do not know how to compute $\Delta$
for an arbitrary target and rely on families where it is known analytically. And
as noted in Section~\ref{sec:overgen}, the innate measure on continuous targets is
not invariant to rescaling.

\paragraph{Future work.} The continuous half of the matched-pair programme,
$(w^\top\bx)^6$ against an additive degree-$6$ polynomial, is the experiment
that would most extend Section~\ref{sec:junta} beyond the
cube. Proposition~\ref{prop:symmetric} offers a family of targets
well beyond parity, including majority and threshold functions of full degree, on
which the prediction can be tested without relying on parity at all. The order of
the evolved representation, used here as a diagnostic, is also a candidate
selective pressure, giving a Boolean counterpart to the Pareto scheme
of~\cite{vladislavleva2009order}. Finally, Section~\ref{sec:multi} traces the
feature-count trade only on parity; multidimensional feature construction
systems in the wild face targets with uneven block structure, and mapping how
the per-feature order and the combination capacity divide the work there is
the natural next step.

\bibliographystyle{unsrt}
\bibliography{references}

\end{document}